%% file: OPPO_arxiv.tex
\documentclass[12pt]{article}
\usepackage{enumerate}
\usepackage[OT1]{fontenc}
\usepackage{amsmath,amssymb}
\usepackage{natbib}
\usepackage[usenames]{color}

\usepackage{smile}
\usepackage{multirow}
\usepackage[breaklinks,colorlinks,
            linkcolor=red,
            anchorcolor=blue,
            citecolor=blue
            ]{hyperref}

\def\given{\,|\,}

\def\tr{\mathop{\text{tr}}\kern.2ex}

\long\def\comment#1{}

\def\tr{\mathop{\text{Tr}}}

\def\cS{{\mathcal{S}}}

\def\cP{{\mathcal{P}}}

\def\tr{{\text{Tr}}}

\def\dr{\displaystyle \rm}

\newcommand{\bel}{\begin{eqnarray}\label}
\newcommand{\eel}{\end{eqnarray}}
\newcommand{\bes}{\begin{eqnarray*}}
\newcommand{\ees}{\end{eqnarray*}}

\newcommand{\red}{\color{black}}

\newcommand{\la}{\langle}
\newcommand{\ra}{\rangle}

\renewcommand{\bd}{\ud}

\usepackage{mathrsfs}

\usepackage{fullpage}

\def \dkl{D_{\rm KL}}

\def\dr{\displaystyle \rm}
\def\reg{\sqrt{d^2H^3T}}
\def\lok{\log(dT/\zeta)}

\usepackage{ifthen}
\newboolean{showblue}
\setboolean{showblue}{false}

\usepackage{hyperref}
\usepackage[english]{babel}
\usepackage{breakcites}
\usepackage{microtype}
\def\##1\#{\begin{align}#1\end{align}}
\def\$#1\${\begin{align*}#1\end{align*}}
\usepackage{enumitem}

\usepackage{setspace} 

\begin{document}

\title{\LARGE Provably Efficient Exploration in Policy Optimization}
\author
{\normalsize
Qi Cai\thanks{Northwestern University; 
\texttt{qicai2022@u.northwestern.edu}}
\qquad Zhuoran Yang\thanks{Princeton University; 
\texttt{zy6@princeton.edu}}
\qquad Chi Jin\thanks{Princeton University; 
\texttt{chij@princeton.edu}}
\qquad Zhaoran Wang\thanks{Northwestern University; 
\texttt{zhaoranwang@gmail.com}}
}
\date{}
\maketitle

\setlength{\abovedisplayskip}{6pt}
\setlength{\belowdisplayskip}{6pt}

\input{abs.tex}

\input{intro.tex}

\input{prelim.tex}
\input{algo.tex}

\input{sketch.tex}

\section{Conclusion}
We study the sample efficiency of policy-based reinforcement learning in the episodic setting of linear MDPs with full-information feedback. We proposed an optimistic variant of the proximal policy optimization algorithm, dubbed as OPPO, which incorporates the principle of ``optimism in the face of uncertainty'' into policy optimization. When applied to the episodic MDP with unknown transition and adversarial reward, OPPO provably achieves a $\reg$-regret up to logarithmic factors, which is near-optimal. To the best of our knowledge, OPPO is the first provably efficient policy optimization algorithm that explicitly incorporates exploration. 

\section*{Acknowledgements}
The authors would like to thank Lingxiao Wang, Wen Sun, and Sham Kakade for pointing out a technical issue in the first version regarding the covering number of value functions in the linear setting. This version has fixed the technical issue with a definition of the linear MDP different from the one in the first version. The authors would also like to thank Csaba Szepesv\'ari, Lin F. Yang, Yining Wang, and Simon S. Du for helpful discussions. Zhaoran Wang acknowledges National Science Foundation (Awards 2048075, 2008827, 2015568, 1934931), Simons Institute (Theory of Reinforcement Learning), Amazon, J.P. Morgan, and Two Sigma for their supports.

\bibliographystyle{ims}
\bibliography{graphbib.bib}

\newpage
\appendix{}
\input{appendix}

\input{aux_lemma}

\input{tab}

\end{document}

%% file: abs.tex
\begin{abstract}
While policy-based reinforcement learning (RL) achieves tremendous successes in practice, it is significantly less understood in theory, especially compared with value-based RL. In particular, it remains elusive how to design a provably efficient policy optimization algorithm that incorporates exploration. To bridge such a gap, this paper proposes an \underline{O}ptimistic variant of the \underline{P}roximal \underline{P}olicy \underline{O}ptimization algorithm (OPPO), which follows an ``optimistic version'' of the policy gradient direction. This paper proves that, in the problem of episodic Markov decision process with linear function approximation, unknown transition, and adversarial reward with full-information feedback, OPPO achieves $\tilde{O}(\reg)$ regret. Here $d$ is the feature dimension, $H$ is the episode horizon, and $T$ is the total number of steps. To the best of our knowledge, OPPO is the first provably efficient~policy optimization algorithm that explores.

\end{abstract}

%% file: intro.tex
\section{Introduction}
Coupled with powerful function approximators such as neural networks, policy optimization plays a key role in the tremendous empirical successes of deep reinforcement learning \citep{silver2016mastering, silver2017mastering, duan2016benchmarking, openai2019dota, wang2018deep}. In sharp contrast, the theoretical understandings of policy optimization remain rather limited from both computational and statistical perspectives. More specifically, from the computational perspective, it remains unclear until recently whether policy optimization converges to the globally optimal policy in a finite number of iterations, even given infinite data. Meanwhile, from the statistical perspective, it still remains unclear how to attain the globally optimal policy with a finite regret or sample complexity. 
 
A line of recent work \citep{fazel2018global, yang2019global, abbasi2019politex, abbasi2019exploration, bhandari2019global, liu2019neural, agarwal2019optimality, wang2019neural} answers the computational question affirmatively by proving that a wide variety of policy optimization algorithms, such as policy gradient (PG)  \citep{williams1992simple, baxter2000direct, sutton2000policy}, natural policy gradient (NPG) \citep{kakade2002natural}, trust-region policy optimization (TRPO) \citep{schulman2015trust}, proximal policy optimization (PPO) \citep{schulman2017proximal}, and actor-critic (AC) \citep{konda2000actor}, converge to the globally optimal policy at sublinear rates of convergence, even when they are coupled with neural networks \citep{liu2019neural, wang2019neural}. However, such computational efficiency guarantees rely on the regularity condition that the state space is already well explored. Such a condition is often implied by assuming either the access to a ``simulator'' (also known as the generative model) \citep{koenig1993complexity, azar2011speedy, azar2012dynamic, azar2012sample, sidford2018near, sidford2018variance, wainwright2019variance} or finite concentratability coefficients \citep{munos2008finite, antos2008fitted, farahmand2010error, tosatto2017boosted, yang2019theoretical, chen2019information}, both of which are often unavailable in practice.


 In a more practical setting, the agent sequentially explores the state space, and meanwhile, exploits the information at hand by taking the actions that lead to higher expected total rewards. Such an exploration-exploitation tradeoff is better captured by the aforementioned statistical question regarding the regret or sample complexity, which remains even more challenging to answer than the computational question. As a result, such a lack of statistical understanding hinders the development of more sample-efficient policy optimization algorithms beyond heuristics. In fact, empirically, vanilla policy gradient is known to exhibit a possibly worse sample complexity than random search \citep{mania2018simple}, even in basic settings such as linear-quadratic regulators. Meanwhile, theoretically, vanilla policy gradient can be shown to suffer from exponentially large variance in the well-known ``combination lock'' setting \citep{kakade2003sample, leffler2007efficient, azar2012dynamic}, which only has a finite state space. 
 
 In this paper, we aim to answer the following fundamental question: 
 \begin{center}
\textit{Can we design a policy optimization algorithm that incorporates exploration
and is provably sample-efficient?}
 \end{center}
To answer this question, we propose the first policy optimization algorithm that incorporates exploration in a principled manner. In detail, we develop an Optimistic variant of the PPO algorithm, namely OPPO. Our algorithm is also closely related to NPG and TRPO. At each update, OPPO solves a Kullback-Leibler (KL)-regularized policy optimization subproblem, where the linear component of the objective function is defined using the action-value function. As is shown subsequently, solving such a subproblem corresponds to one iteration of infinite-dimensional mirror descent \citep{nemirovsky1983problem} or dual averaging \citep{xiao2010dual}, where the action-value function plays the role of the gradient. To encourage exploration, we explicitly incorporate a bonus function into the action-value function, which quantifies the uncertainty that arises from only observing finite historical data. Through uncertainty quantification, such a bonus function ensures the (conservative) optimism of the updated policy. Based on NPG, TRPO, and PPO, OPPO only augments the action-value function with the bonus function in an additive manner, which makes it easily implementable in practice.

Theoretically, we establish the sample efficiency of OPPO in an episodic setting of Markov decision processes (MDPs) with full-information feedback, where the transition dynamics are linear in features \citep{yang2019sample, yang2019reinforcement, jin2019provably, ayoub2020model, zhou2020provably}. In particular, we allow the transition dynamics to be nonstationary within each episode. See also the work of \cite{du2019good, van2019comments, lattimore2019learning} for a related discussion on the necessity of the linear representation. In detail, we prove that OPPO attains a $\reg$-regret up to logarithmic factors, where $d$ is the feature dimension, $H$ is the episode horizon, and $T$ is the total number of steps taken by the agent. Note that such a regret does not depend on the numbers of states and actions, and therefore, allows them to be even infinite. In particular, OPPO attains such a regret without knowing the transition dynamics or accessing a ``simulator".
Moreover, we prove that, even when the reward functions are adversarially chosen across the episodes, OPPO attains the same regret in terms of competing with the globally optimal policy in hindsight \citep{cesa2006prediction, bubeck2012regret}. In comparison, existing algorithms based on value iteration, e.g., optimistic least-squares value iteration (LSVI) \citep{jin2019provably}, do not allow adversarially chosen reward functions. Such a notion of robustness partially justifies the empirical advantages of KL-regularized policy optimization \citep{neu2017unified, geist2019theory}. To the best of our knowledge, OPPO is the first provably sample-efficient policy optimization algorithm that incorporates exploration.

\subsection{Related Work}

Our work is based on the aforementioned line of recent work \citep{fazel2018global, yang2019global, abbasi2019politex, abbasi2019exploration, bhandari2019global, liu2019neural, agarwal2019optimality, wang2019neural} on the computational efficiency of policy optimization, which covers PG, NPG, TRPO, PPO, and AC. In particular, OPPO is based on PPO (and similarly, NPG and TRPO), which is shown to converge to the globally optimal policy at sublinear rates in tabular and linear settings, as well as nonlinear settings involving neural networks \citep{liu2019neural, wang2019neural}. However, without assuming the access to a ``simulator'' or finite concentratability coefficients, both of which imply that the state space is already well explored, it remains unclear whether any of such algorithms is sample-efficient, that is, attains a finite regret or sample complexity. In comparison, by incorporating uncertainty quantification into the action-value function at each update, which explicitly encourages exploration, OPPO not only attains the same computational efficiency as NPG, TRPO, and PPO, but is also shown to be sample-efficient with a $\reg$-regret up to logarithmic factors. 

Our work is closely related to another line of work \citep{even2009online, yu2009markov, neu2010bandit, neu2010online, zimin2013online, neu2012adversarial, rosenberg2019online, rosenberg2019bandit} on online MDPs with adversarially chosen reward functions, which mostly focuses on the tabular setting. 
\begin{itemize}
\item Assuming the transition dynamics are known and the full information of the reward functions is available, the work of \cite{even2009online} establishes a $\sqrt{\tau^2 T\cdot\log |\cA|}$-regret, where $\cA$ is the action space, $|\cA|$ is its cardinality, and $\tau$ upper bounds the mixing time of the MDP. See also the work of \cite{yu2009markov}, which establishes a $T^{2/3}$-regret in a similar setting. 
\item Assuming the transition dynamics are known but only the bandit feedback of the received rewards is available, the work of \cite{neu2010bandit, neu2010online, zimin2013online} establishes an $H^2\sqrt{|\cA|T}/\beta$-regret \citep{neu2010online}, a $T^{2/3}$-regret \citep{neu2010bandit}, and a $\sqrt{H  |\cS| |\cA| T}$-regret \citep{zimin2013online}, respectively, all up to logarithmic factors. Here $\cS$ is the state space and $|\cS|$ is its cardinality. In particular, it is assumed by \cite{neu2010online} that, with probability at least $\beta$, any state is reachable under any policy. 
\item Assuming the full information of the reward functions is available but the transition dynamics are unknown, the work of \cite{neu2012adversarial, rosenberg2019online} establishes an $H|\cS||\cA| \sqrt{T}$-regret \citep{neu2012adversarial} and an $H  |\cS| \sqrt{|\cA| T}$-regret \citep{rosenberg2019online}, respectively, both up to logarithmic factors. 
\item Assuming the transition dynamics are unknown and only the bandit feedback of the received rewards is available, the recent work of \cite{rosenberg2019bandit} establishes an $H|\cS|\sqrt{|\cA|T}/\beta$-regret up to logarithmic factors. In particular, it is assumed by \cite{rosenberg2019bandit} that, with probability at least $\beta$, any state is reachable under any policy. Without such an assumption, an $H^{3/2} |\cS| |\cA|^{1/4} T^{3/4}$-regret is established. 
\end{itemize}
In the latter two settings with unknown transition dynamics, all the existing algorithms \citep{neu2012adversarial, rosenberg2019online, rosenberg2019bandit} follow the gradient direction with respect to the visitation measure, and thus, differ from most practical policy optimization algorithms. In comparison, OPPO is not restricted to the tabular setting and indeed follows the gradient direction with respect to the policy. OPPO is simply an optimistic variant of NPG, TRPO, and PPO, which makes it also a practical policy optimization algorithm.
{\red In particular, when specialized to the tabular setting, our setting corresponds to the third setting with $d = |\cS|^2 |\cA|$, where OPPO attains an $H^{3/2} |\cS|^{2} |\cA|  \sqrt{T}$-regret up to logarithmic factors. 
}

Broadly speaking, our work is related to a vast body of work on value-based reinforcement learning in tabular \citep{jaksch2010near, osband2014generalization, osband2016lower, azar2017minimax, dann2017unifying, strehl2006pac, jin2018q} and linear settings \citep{yang2019sample, yang2019reinforcement, jin2019provably, ayoub2020model, zhou2020provably}, as well as nonlinear settings involving general function approximators \citep{wen2017efficient, jiang2017contextual, du2019provably, dong2019sqrt}. In particular, our setting is the same as the linear setting studied by \cite{ayoub2020model, zhou2020provably}, which generalizes the one proposed by \cite{yang2019reinforcement}. We remark that our setting differs from the linear setting studied by \cite{yang2019sample, jin2019provably}. It can be shown that the two settings are incomparable in the sense that one does not imply the other \citep{zhou2020provably}. Also, our setting is related to the low-Bellman-rank setting studied by \cite{jiang2017contextual, dong2019sqrt}. In comparison, we focus on policy-based reinforcement learning, which is significantly less studied in theory. In particular, compared with the work of \cite{yang2019sample, yang2019reinforcement, jin2019provably, ayoub2020model, zhou2020provably}, which focuses on value-based reinforcement learning, OPPO attains the same $\sqrt{T}$-regret even in the presence of adversarially chosen reward functions. Compared with optimism-led iterative value-function elimination (OLIVE) \citep{jiang2017contextual, dong2019sqrt}, which handles the more general low-Bellman-rank setting but is only sample-efficient, OPPO simultaneously attains computational efficiency and sample efficiency in the linear setting. Despite the differences between policy-based and value-based reinforcement learning, our work shows that the general principle of ``optimism in the face of uncertainty'' \citep{auer2002finite, bubeck2012regret} can be carried over from existing algorithms based on value iteration, e.g., optimistic LSVI, into policy optimization algorithms, e.g., NPG, TRPO, and PPO, to make them sample-efficient, which further leads to a new general principle of ``conservative optimism in the face of uncertainty and adversary'' that additionally allows adversarially chosen reward functions.


 

%% file: prelim.tex
\subsection{Notation}
We denote by $\|\cdot\|_2$ the $\ell_2$-norm of a vector or the spectral norm of a matrix. We denote by $\Delta(\cA)$ the set of probability distributions on a set $\cA$ and correspondingly define
\$
\Delta(\cA\,|\,\cS, H) &= \bigl\{ \{\pi_h(\cdot\,|\,\cdot)\}_{h=1}^H: \pi_h(\cdot\,|\,x)\in\Delta(\cA) ~\text{for any $x\in\cS$ and $h\in[H]$} \bigr\}
\$
for any set $\cS$ and $H\in\mathbb{Z}_{+}$. For $p_1,p_2\in\Delta(\cA)$, we denote by $\dkl(p_1\,\|\,p_2)$ the KL-divergence,
\$
\dkl(p_1\,\|\,p_2) = \sum_{a\in\cA} p_1(a)\log\frac{p_1(a)}{p_2(a)}. 
\$
Throughout this paper, we denote by $C, C', C'', \ldots$ absolute constants whose values can vary from line by line. 

\section{Preliminaries}
\subsection{MDPs with Adversarial Rewards}
In this paper, we consider an episodic MDP $(\cS,\cA, H, \cP, r)$, where $\cS$ and $\cA$ are the state and action spaces, respectively, $H$ is the length of each episode, $\cP_h(\cdot\,|\,\cdot,\cdot)$ is the transition kernel from a state-action pair to the next state at the $h$-th step of each episode, and $r^k_h:\cS\times\cA\rightarrow [0,1]$ is the reward function at the $h$-th step of the $k$-th episode. We assume that the reward function is deterministic, which is without loss of generality, as our subsequent regret analysis readily generalizes to the setting where the reward function is stochastic. 


At the beginning of the $k$-th episode, the agent determines a policy $\pi^k = \{\pi^k_h\}_{h=1}^H\in\Delta(\cA\,|\,\cS, H)$. We assume that the initial state $x^k_1$ is fixed to $x_1\in\cS$ across all the episodes, which is without loss of generality, as our subsequent regret analysis readily generalizes to the setting where $x^k_1$ is sampled from a fixed distribution across all the episodes. Then the agent iteratively interacts with the environment as follows. At the $h$-th step, the agent receives a state $x^k_h$ and takes an action following $a^k_h\sim\pi^k_h(\cdot\,|\,x^k_h)$. Subsequently, the agent receives a reward $r^k_h(x^k_h,a^k_h)$ and the next state following $x^k_{h+1}\sim\cP_h(\cdot\,|\,x^k_{h},a^k_{h})$. The $k$-th episode ends after the agent receives the last reward $r^k_H(x^k_H,a^k_H)$.



We allow the reward function $r^k=\{r^k_h\}_{h=1}^H$ to be adversarially chosen by the environment at the beginning of the $k$-th episode, which can depend on the  $(k-1)$ historical trajectories. The reward function $r^k_h$ is revealed to the agent after it takes the action $a^k_h$ at the state $x^k_h$, which together determine the received reward $r^k_h(x^k_h,a^k_h)$. We define the regret in terms of competing with the globally optimal policy in hindsight \citep{cesa2006prediction, bubeck2012regret} as
\#\label{regret}
\text{Regret}(T) = \max_{\pi \in \Delta(\cA\,|\,\cS,H)} \sum_{k=1}^K \bigl(V^{\pi,k}_1(x^k_1) - V^{\pi^k,k}_1(x^k_1)\bigr),
\#
where $T=HK$ is the total number of steps taken by the agent in all the $K$ episodes. For any policy $\pi = \{\pi_h\}_{h=1}^H\in\Delta(\cA\,|\,\cS, H)$, the value function $V^{\pi,k}_h: \cS \rightarrow \RR$ associated with the reward function $r^k=\{r^k_h\}_{h=1}^H$ is defined by
\#\label{eq:vfunc}
V^{\pi,k}_h(x) =\EE_{\pi} \Bigl[ \sum_{i=h}^H r^k_i(x_i,a_i) \,\Big|\, x_h=x \Bigr].
\#
Here we denote by $\EE_{\pi}[\cdot]$ the expectation with respect to the randomness of the state-action sequence $\{(x_h,a_h)\}_{h=1}^H$, where the action $a_h$ follows the policy $\pi_h(\cdot \,|\,x_h)$ at the state $x_h$ and the next state $x_{h+1}$ follows the transition dynamics $\cP_h(\cdot\,|\,x_h,a_h)$.
 Correspondingly, we define the action-value function (also known as the Q-function) $Q^{\pi,k}_h: \cS \times \cA \rightarrow \RR$ as 
\#\label{eq:qfunc}
Q^{\pi,k}_h(x,a) =\EE_{\pi} \Bigl[ \sum_{i=h}^H r^k_i(x_i,a_i) \,\Big|\, x_h=x,\,a_h=a  \Bigr].
\#
By the definitions in \eqref{eq:vfunc} and \eqref{eq:qfunc}, we have the following Bellman equation, 
\#\label{eq:w1204}
V^{\pi,k}_h = \la Q^{\pi,k}_h, \pi_h \ra_\cA,\quad 
Q^{\pi,k}_h = r^k_h + \mathbb{P}_h V^{\pi,k}_{h+1}.
\#
Here $\la\cdot,\cdot\ra_{\cA}$ denotes the inner product over $\cA$, where the subscript is omitted subsequently if it is clear from the context. Also, $\mathbb{P}_h$ is the operator form of the transition kernel $\cP_h(\cdot\,|\,\cdot,\cdot)$, which is defined by
\#\label{eq:w256c}
(\mathbb{P}_h f)(x,a) = \EE[ f(x') \,|\, x'\sim \cP_h(\cdot\,|\,x,a)  ]
\#
for any function $f:\cS\rightarrow \RR$. By allowing the reward function to be adversarially chosen in each episode, our setting generalizes the stationary setting commonly adopted by the existing work on value-based reinforcement learning \citep{jaksch2010near, osband2014generalization, osband2016lower, azar2017minimax, dann2017unifying, strehl2006pac, jin2018q, jin2019provably, yang2019sample, yang2019reinforcement}, where the reward function is fixed across all the episodes.

\subsection{Linear Function Approximations}
We consider the linear setting where the transition dynamics are linear in a feature map, which is formalized in the following assumption. 

\begin{assumption}[Linear MDP \citep{ayoub2020model, zhou2020provably}]\label{linearmdp} We assume that the MDP $(\cS,\cA, H, \cP, r)$ is a linear MDP with the known feature map $\psi: \cS\times\cA\times\cS \rightarrow\RR^d$, that is, for any $h\in[H]$, there exists $\theta_h\in\RR^d$ with $\|\theta_h\|_2\le \sqrt{d}$ such that
\$
\cP_h(x'\,|\,x,a)=\psi(x,a,x')^\top\theta_h
\$
for any $(x,a,x')\in\cS\times\cA\times\cS$. Also, we assume that
\$
\Bigl\| \int_\cS \psi(x,a,x')\cdot V(x')\bd x' \Bigr\|_2 \le \sqrt{d}H
\$
 for any $(x,a)\in\cS\times\cA$ and $V:\cS\rightarrow[0,H]$.
\end{assumption}

See \cite{ayoub2020model, zhou2020provably} for various examples of linear MDPs, including the one proposed by \cite{yang2019reinforcement}.  In particular, a tabular MDP corresponds to the linear MDP with $d=|\cS|^2|\cA|$ and the feature vector $\psi(x, a,x')$ being the canonical basis $e_{(x, a,x')}$ of $\RR^{|\cS|^2|\cA|}$. See also \cite{du2019good, van2019comments, lattimore2019learning} for a related discussion on the necessity of the linear representation.

{\red We remark that \cite{yang2019sample, jin2019provably}
study another variant of linear MDPs, where the transition kernel can be written as $\cP_{h} (x'\given x, a) = \varphi(x,a)^\top \mu_h(x')$ for any $h \in [H]$ and $(x,a,x') \in \cS\times \cA \times \cS$.  Here $\varphi \colon \cS \times \cA \rightarrow \RR^d$ is a known feature map and $\mu_h\colon \cS\rightarrow \RR^d$ is an unknown function on $\cS$ for any $h \in [H]$. 
Although the variant of linear MDPs defined in Assumption \ref{linearmdp} and the one studied by \cite{yang2019sample, jin2019provably} both cover the tabular setting and the one proposed by \cite{yang2019reinforcement} as special cases, they are two different definitions of linear MDPs as their feature maps $\psi (\cdot , \cdot, \cdot)$ and $\varphi(\cdot, \cdot)$ are defined on different domains. It can be shown that the two variants are incomparable in the sense that one does not imply the other \citep{zhou2020provably}.

} 

%% file: algo.tex
\section{Algorithm and Theory}\label{results}
\subsection{Optimistic PPO (OPPO)}\label{secalgo}
We present Optimistic PPO (OPPO) in Algorithm~\ref{ppoalgo}, which involves a policy improvement step and a policy evaluation step.

\vspace{4pt}
\noindent
{\bf Policy Improvement Step.} In the $k$-th episode, OPPO updates $\pi^k$ based on $\pi^{k-1}$ (Lines 4-9 of Algorithm~\ref{ppoalgo}). In detail, we define the following linear function of the policy $\pi \in\Delta(\cA\,|\,\cS, H)$,
\#\label{linearappr}
&L_{k-1}(\pi) =  V^{\pi^{k-1},k-1}_1(x^k_1) +  \EE_{\pi^{k-1}} \Bigl[\sum_{h=1}^H \la Q^{\pi^{k-1},k-1}_h(x_h, \cdot), 
  \pi_h(\cdot\,|\,x_h) -  \pi^{k-1}_h(\cdot\,|\,x_h)   \ra\,\Big|\, x_1 = x^k_1 \Bigr], 
  \#
which is a local linear approximation of $V_1^{\pi,  k-1}(x^k_1)$ at $\pi^{k-1}$ \citep{schulman2015trust, schulman2017proximal}. In particular, we have that $L_{k-1}(\pi^{k-1}) = V^{\pi^{k-1}, k-1}_1(x^k_1)$. The policy improvement step is defined~by 
\#\label{1014522}
&\pi^k \leftarrow \argmax_{\pi\in\Delta(\cA\,|\,\cS,H)} L_{k-1}(\pi)  - \alpha^{-1} \cdot  \EE_{\pi^{k-1}}[ \tilde{D}_{\text{KL}}(\pi\,\|\,\pi^{k-1})\,|\, x_1 = x^k_1], \\ 
& \text{where~~}\tilde{D}_{\text{KL}}(\pi\,\|\,\pi^{k-1})=\sum_{h=1}^H \dkl\bigl( \pi_h(\cdot\,|\,x_h) \,\big\|\, \pi^{k-1}_h(\cdot\,|\,x_h)\bigr). \notag
\#
 Here the KL-divergence regularizes $\pi$ to be close to $\pi^{k-1}$ so that $L_{k-1}(\pi)$ well approximates $V^{\pi,k-1}_1(x^k_1)$, which further ensures that the updated policy $\pi^{k}$ improves the expected total reward (associated with the reward function $r^{k-1}$) upon $\pi^{k-1}$. Also, $\alpha > 0$ is the stepsize, which is specified in Theorem \ref{10061231}. By executing the updated policy $\pi^k$, the agent  receives the state-action sequence $\{(x^k_h, a^k_h)\}_{h=1}^H$ and observes the reward function $r^k$, which together determine the received rewards $\{r^k_h(x^k_h, a^k_h)\}_{h=1}^H$.

The policy improvement step defined in \eqref{1014522} corresponds to one iteration of NPG \citep{kakade2002natural}, TRPO \citep{schulman2015trust}, and PPO \citep{schulman2017proximal}. In particular, PPO solves the same KL-regularized policy optimization subproblem as in \eqref{1014522} at each iteration, while TRPO solves  an equivalent KL-constrained subproblem. In the special case where the reward function $r^{k-1}_h$ is linear in the feature map $\phi^{k-1}_h$ defined subsequently, which implies that the Q-function $Q^{\pi^{k-1},k-1}_h$ is also linear in $\phi^{k-1}_h$, the updated policy $\pi^k$ can be equivalently obtained by one iteration of NPG when the policy is parameterized by an energy-based distribution \citep{agarwal2019optimality, wang2019neural}. Such a policy improvement step can also be cast as one iteration of infinite-dimensional mirror descent \citep{nemirovsky1983problem} or dual averaging \citep{xiao2010dual}, where the Q-function plays the role of  the gradient \citep{liu2019neural, wang2019neural}.

\begin{algorithm}[t]
\caption{Optimistic PPO (OPPO)}
\begin{algorithmic}[1]
\STATE Initialize $\{\pi^0_h(\cdot \,|\, \cdot)\}_{h=1}^H$ as uniform distributions on $\cA$ and $\{Q^0_h(\cdot, \cdot)\}_{h=1}^H$ as zero functions.\label{line:winit}
\STATE \textbf{For} episode $k=1,2,\ldots, K$ \textbf{do}
\STATE \hspace{0.15in} Receive the initial state $x_1^k$.
\STATE \hspace{0.15in} \textbf{For} step {$h=1, 2, \ldots, H$} \textbf{do} \label{line:pis-start}
\STATE \hspace{0.30in} Update the policy by
\STATE \hspace{0.40in} $\pi^k_h(\cdot\,|\,\cdot) \propto \pi^{k-1}_h(\cdot\,|\,\cdot) \cdot \exp\{\alpha\cdot Q^{k-1}_h(\cdot,\cdot)\}$. \label{policyupdateline}
\STATE \hspace{0.30in} Take the action following $a^k_{h}\sim\pi^k_h(\cdot\,|\,x_h^k)$.
\STATE \hspace{0.30in} Observe the reward function $r^k_{h}(\cdot,\cdot)$.
\STATE \hspace{0.30in}  Receive the next state $x^k_{h+1}$. \label{line:pis-end}
\STATE \hspace{0.15in} Initialize $V^k_{H+1}(\cdot)$ as a zero function.
\STATE \hspace{0.15in} \textbf{For} step {$h=H, H-1,\ldots, 1$} \textbf{do}\label{line:pes-start}
\STATE \hspace{0.30in} $\Lambda^k_h\leftarrow \sum_{\tau=1}^{k-1} \phi^\tau_h(x^\tau_h, a^\tau_h)\phi^\tau_h(x^\tau_h, a^\tau_h)^\top + \lambda\cdot I$.
\STATE \hspace{0.30in} $w^k_h\leftarrow(\Lambda^k_h)^{-1} \sum_{\tau=1}^{k-1}\phi^\tau_h(x^\tau_h, a^\tau_h) \cdot V^\tau_{h+1}(x^\tau_{h+1})$.
\STATE \hspace{0.30in} $\phi^k_h(\cdot, \cdot) \leftarrow \int_\cS \psi(\cdot,\cdot,x') \cdot V^k_{h+1}(x')\bd x'$.
\STATE \hspace{0.30in} $\Gamma^k_h(\cdot,\cdot)\leftarrow \beta\cdot [\phi^k_h(\cdot,\cdot)^\top(\Lambda^k_h)^{-1}\phi^k_h(\cdot,\cdot)]^{1/2}$.
\STATE \hspace{0.30in} $\overbar{Q}^k_h(\cdot,\cdot)\leftarrow r^k_h(\cdot,\cdot) + \phi^k_h(\cdot,\cdot)^\top w^k_h + \Gamma^k_h(\cdot,\cdot)$.
\STATE \hspace{0.30in} $Q^k_h(\cdot, \cdot)\leftarrow \min\{ \overbar{Q}^k_h(\cdot,\cdot), H -h+1\}^+$.
\STATE \hspace{0.30in} $V^k_h(\cdot)\leftarrow \la Q^k_h(\cdot,\cdot), \pi^k_h(\cdot\,|\,\cdot) \ra_\cA$.
\end{algorithmic}\label{ppoalgo}
\end{algorithm}

The updated policy $\pi^k$ obtained in \eqref{1014522} takes the following closed form,
\#\label{eq:w1045}
\pi^{k}_h(\cdot\,|\,x) \propto \pi^{k-1}_h(\cdot\,|\,x) \cdot \exp\bigl(\alpha\cdot Q^{\pi^{k-1}, k-1}_h(x, \cdot)\bigr)
\#
for any $h \in [H]$ and $x\in\cS$. However, the Q-function $Q^{\pi^{k-1}, k-1}_h$ remains to be estimated through the subsequent policy evaluation step. We denote by $Q^{k-1}_h$ the estimated Q-function, which replaces the Q-function $Q^{\pi^{k-1}, k-1}_h$ in \eqref{linearappr}-\eqref{eq:w1045} and is correspondingly used in Line 6 of Algorithm~\ref{ppoalgo}.

\vspace{4pt}
\noindent
{\bf Policy Evaluation Step.} At the end of the $k$-th episode, OPPO evaluates the policy $\pi^{k}$ based on the $(k-1)$ historical trajectories (Lines 11-18 of Algorithm~\ref{ppoalgo}). In detail, for any $h\in [H]$, we define the empirical mean-squared Bellman error (MSBE) \citep{sutton2018reinforcement} as
\#\label{eq:w230a}
&M^k_h(w) = \sum_{\tau=1}^{k-1} \bigl(V^{\tau}_{h+1}(x^{\tau}_{h+1}) - \phi^\tau_h(x^\tau_h, a^\tau_h)^\top w\bigr)^2,\\
& \text{where~~} 
\phi^\tau_h(\cdot,\cdot)=\int_\cS \psi(\cdot,\cdot,x') \cdot V^\tau_{h+1}(x')\bd x',  \notag\\
&\quad\qquad V^\tau_{h+1}(\cdot) = \la Q_{h+1}^\tau (\cdot , \cdot ), \pi_{h+1}^\tau (\cdot \given \cdot )\ra_\cA,  \notag
\#
while we initialize $V^\tau_{H+1}$ as a zero function on $\cS$. The policy evaluation step is defined by iteratively updating the estimated Q-function $Q^k = \{Q^k_h\}_{h=1}^H$ associated with the reward function $r^{k} = \{r^k_h\}_{h=1}^H$ by 
\#\label{eq:w1150}
 & w^k_h  \leftarrow \argmin_{w \in \RR^d} M^k_h(w) + \lambda \cdot \|w\|_2^2,  \notag\\
&  Q^{k}_h(\cdot,\cdot)  \leftarrow  \min\{r^k_h(\cdot,\cdot) + \phi^k_h(\cdot,\cdot)^\top w^k_h + \Gamma^k_h(\cdot,\cdot), H-h+1\}^+  
\#
in the order of $h=H, H-1, \ldots, 1$. Here $\lambda > 0$ is the regularization parameter, which is specified in Theorem \ref{10061231}. Also, $\Gamma^k_h: \cS\times \cA\rightarrow \RR^+$ is a bonus function, which quantifies the uncertainty in estimating the Q-function $Q^{\pi^{k}, k}_h$ based on only finite historical data. In particular, the weight vector $w^k_h$ obtained in \eqref{eq:w1150} and the bonus function $\Gamma^k_h$ take the following closed forms, 
\#\label{1015509}
&w^k_h = (\Lambda^k_h)^{-1} \Bigl( \sum_{\tau=1}^{k-1} \phi^\tau_h(x_h^\tau,a_h^\tau)\cdot V_{h+1}^\tau(x_{h+1}^\tau) \Bigr),\notag\\
& \Gamma^k_h(\cdot, \cdot)=\beta\cdot \bigl(\phi^k_h(\cdot,\cdot)^\top(\Lambda^k_h)^{-1} \phi^k_h(\cdot,\cdot)\bigr)^{1/2},
\\
&\text{where~~}\Lambda^k_h=\sum_{\tau=1}^{k-1} \phi^\tau_h(x_h^\tau,a_h^\tau)\phi^\tau_h(x_h^\tau,a_h^\tau)^\top + \lambda\cdot I.\notag
\#
Here $\beta > 0$ scales with $d$, $H$, and $K$, which is specified in Theorem \ref{10061231}. 

The policy evaluation step defined in \eqref{eq:w1150} corresponds to one iteration of least-squares temporal difference (LSTD) \citep{bradtke1996linear, boyan2002technical}. In particular, as we have 
\$
\EE[ V^{\tau}_{h+1}(x') \,|\, x'\sim \cP_h(\cdot\,|\,x,a)  ] = (\mathbb{P}_h V^{\tau}_{h+1})(x,a)  
\$
for any $\tau \in [k-1]$ and $(x,a)\in\cS\times\cA$ in the empirical MSBE defined in \eqref{eq:w230a}, ${\phi^{k\top}_h} w^k_h$ in \eqref{eq:w1150} is an estimator of $\mathbb{P}_h V^k_{h+1}$ in the Bellman equation defined in \eqref{eq:w1204} (with $V^{\pi^k, k}_{h+1}$ replaced by $V^k_{h+1}$). Meanwhile, we construct the bonus function $\Gamma^k_h$ according to \eqref{1015509} so that ${\phi^{k\top}_h} w^k_h + \Gamma^k_h$ is an upper confidence bound (UCB), that is, it holds that 
\$
\phi^k_h(\cdot,\cdot)^\top w^k_h + \Gamma^k_h(\cdot,\cdot) \geq (\mathbb{P}_h V^k_{h+1})(\cdot,\cdot)
\$ with high probability, which is subsequently characterized in Lemma \ref{10061229}. Here the inequality holds uniformly for any $(h,k)\in [H]\times [K]$ and $(x, a) \in \cS \times \cA$. As the fact that $r^k_h \in [0,1]$ for any $h \in [H]$ implies that $Q^{\pi^k,k}_{h} \in [0, H-h+1]$, we truncate $Q^k_h$ to the range $[0, H-h+1]$ in \eqref{eq:w1150}, which is correspondingly used in Line 17 of Algorithm~\ref{ppoalgo}.

\subsection{Regret Analysis}\label{secmd}
We establish an upper bound of the regret of OPPO (Algorithm~\ref{ppoalgo}) in the following theorem. Recall that the regret is defined in \eqref{regret} and $T = HK$ is the total number of steps taken by the agent, where $H$ is the length of each episode and $K$ is the total number of episodes. Also, $|\cA|$ is the cardinality of $\cA$ and $d$ is the dimension of the feature map $\psi$.

\begin{theorem}[Total Regret]\label{10061231} 
Let $\alpha=\sqrt{2\log{|\cA|}/(HT)}$ in \eqref{1014522} and Line 6 of Algorithm~\ref{ppoalgo}, $\lambda=1$ in \eqref{eq:w1150} and Line 12 of Algorithm~\ref{ppoalgo}, and $\beta=C\sqrt{dH^2\cdot\lok}$ in \eqref{1015509} and Line 15 of Algorithm~\ref{ppoalgo}, where $C>1$ is an absolute constant and $\zeta\in (0,1]$. Under Assumption \ref{linearmdp} and the assumption that $\log |\cA| = O(d^2\cdot [\lok]^2)$, the regret of OPPO satisfies 
\$
\text{Regret}(T) \le C'\reg\cdot \lok
\$
with probability at least $1-\zeta$, where $C' > 0$ is an absolute constant. 
\end{theorem}
\begin{proof}
See Section \ref{sketch} for a proof sketch and Appendix \ref{1008755} for a detailed proof.
\end{proof}

Theorem \ref{10061231} proves that OPPO attains a $\reg$-regret up to logarithmic factors, where the dependency on the total number of steps $T$ is optimal. 
In the stationary setting where the reward function and initial state are fixed across all the episodes, such a regret translates to a $d^2 H^4/ \varepsilon^2$-sample complexity (up to logarithmic factors) following the argument of \cite{jin2018q} (Section 3.1). Here $\varepsilon > 0$ measures the suboptimality of the obtained policy $\pi^{k}$ in the following sense,
\$
\max_{\pi \in \Delta(\cA\,|\,\cS,H)} V^{\pi}_1(x_1) - V^{\pi^{k}}_1(x_1) \leq \varepsilon,
\$
where $k$ is sampled from $[K]$ uniformly at random. Here we denote the value function by $V^{\pi}_1 = V^{\pi, k}_1$ and the initial state by $x_1 = x_1^k$ for any $k\in [K]$, as the reward function and initial state are fixed across all the episodes. Moreover, compared with the work of \cite{yang2019sample, yang2019reinforcement, jin2019provably, ayoub2020model, zhou2020provably}, OPPO additionally allows adversarially chosen reward functions without exacerbating the regret, which leads to a notion of robustness.  
{\red Also, as a tabular MDP satisfies Assumption \ref{linearmdp} with $d = | \cS|^2 |\cA|$ and $\psi$ being the canonical basis of $\RR^d$, 
Theorem  \ref{10061231} yields an $| \cS|^2 |\cA| \sqrt{H^3 T} $-regret in the tabular setting.}
Our subsequent discussion intuitively explains how OPPO achieves such a notion of robustness while attaining the $\reg$-regret (up to logarithmic factors).


\vspace{4pt}
\noindent
{\bf Discussion of Mechanisms.} In the sequel, we consider the ideal setting where the transition dynamics are known, which, by the Bellman equation defined in \eqref{eq:w1204}, allows us to access the Q-function $Q_h^{\pi, k}$ for any policy $\pi$ and $(h,k)\in [H]\times [K]$ once given the reward function $r^k$. The following lemma connects the difference between two policies to the difference between their expected total rewards through the Q-function.

\begin{lemma}[Performance Difference]\label{pdlm}
 For any policies $\pi, \pi'\in\Delta(\cA\,|\,\cS,H)$ and $k\in[K]$, it holds that
\#\label{pd}
&V^{\pi',k}_1(x^k_1) - V^{\pi,k}_1(x^k_1) = \EE_{\pi'}\Bigl[\sum_{h=1}^H \la Q^{\pi,k}_h(x_h,\cdot),  \pi'_h(\cdot\,|\,x_h)  - \pi_h(\cdot\,|\,x_h)  \ra \,\Big|\, x_1 = x^k_1\Bigr].  
\#
\end{lemma}

\begin{proof}
See Appendix \ref{pdlmpf} for a detailed proof.
\end{proof}


For notational simplicity, we omit the conditioning on $x_1 = x^k_1$, e.g.,  in \eqref{pd} of Lemma \ref{pdlm}, subsequently. The following lemma characterizes the policy improvement step defined in \eqref{1014522}, where the updated policy $\pi^k$ takes the closed form in \eqref{eq:w1045}.

\begin{lemma}[One-Step Descent]\label{mdlm}
For any distributions $p^*, p\in\Delta(\cA)$, state $x\in\cS$, and function $Q:\cS\times\cA\rightarrow[0,H]$, it holds for $p' \in\Delta(\cA)$ with $p'(\cdot)\propto p(\cdot)\cdot \exp\{\alpha\cdot Q(x,\cdot)\}$ that
\$
&\la Q(x,\cdot), p^*(\cdot) - p(\cdot) \ra  \le \alpha H^2/2+  \alpha^{-1}\cdot \Bigl(\dkl\bigl( p^*(\cdot) \,\big\|\, p(\cdot)\bigr) - 
\dkl\bigl( p^*(\cdot)\,\big\|\, p'(\cdot)\bigr)\Bigr).
\$
\end{lemma}
\begin{proof}
See Appendix \ref{mdlmpf} for a detailed proof.
\end{proof}

Corresponding to the definition of the regret in \eqref{regret}, we define the globally optimal policy in hindsight \citep{cesa2006prediction, bubeck2012regret} as 
\#\label{eq:wpistar}
\pi^* = \argmax_{\pi \in \Delta(\cA\,|\,\cS,H)} \sum_{k=1}^K V^{\pi,k}_1(x^k_1),
\# 
which attains a zero-regret. In the ideal setting where the Q-function $Q^{\pi^k,k}_h$ associated with the reward function $r^k$ is known and the updated policy $\pi^{k+1}_h$ takes the closed form in \eqref{eq:w1045}, Lemma \ref{mdlm} implies 
\#\label{1008704}
&\la Q^{\pi^k,k}_h(x,\cdot), \pi^*_h(\cdot\,|\,x) - \pi^k_h(\cdot\,|\,x) \ra \notag\\
& \quad\le \alpha H^2/2 +  \alpha^{-1}\cdot\Bigl(\dkl\bigl( \pi^*_h(\cdot\,|\,x) \,\big\|\, \pi^k_h(\cdot\,|\,x)\bigr)    - 
\dkl\bigl( \pi^*_h(\cdot\,|\,x) \,\big\|\, \pi^{k+1}_h(\cdot\,|\,x)\bigr)\Bigr)
\#
for any $(h,k)\in[H]\times[K]$ and $x\in\cS$. Combining \eqref{1008704} with Lemma \ref{pdlm}, we obtain
\#\label{1015127}
\text{Regret}(T) & = \sum_{k=1}^K \bigl(V^{\pi^*,k}_1(x^k_1) - V^{\pi^k,k}_1 (x^k_1) \bigr)\notag\\
&=\EE_{\pi^*}\Bigl[ \sum_{k=1}^K \sum_{h=1}^H \la Q^{\pi^k,k}_h(x_h,\cdot), \pi^*_h(\cdot\,|\,x_h) - \pi^k_h(\cdot\,|\,x_h)  \ra \Bigr] \notag\\
&\le \alpha H^3 K/2  + \alpha^{-1}\cdot \sum_{h=1}^H \EE_{\pi^*}  \bigl[ \dkl\bigl( \pi^*_h(\cdot\,|\,x_h) \,\big\|\, \pi^1_h(\cdot\,|\,x_h)\bigr) \bigr] \notag\\
&\le \alpha H^3 K/2 + \alpha^{-1} H\cdot \log|\cA|.
\#
Here the first inequality follows from telescoping the right-hand side of \eqref{1008704} across all the episodes and the fact that the KL-divergence is nonnegative. Also, the second inequality follows from the initialization of the policy and Q-function in Line 1 of Algorithm~\ref{ppoalgo}. Setting $\alpha=\sqrt{2\log|\cA|/ (HT)}$ in \eqref{1015127}, we establish a $\sqrt{H^3T\cdot \log|\cA|}$-regret in the ideal setting.

Such an ideal setting demonstrates the key role of the KL-divergence in the policy improvement step defined in \eqref{1014522}, where $\alpha > 0$ is the stepsize. Intuitively, without the KL-divergence, that is, setting $\alpha\rightarrow\infty$, the upper bound of the regret on the right-hand side of \eqref{1015127} tends to infinity. In fact, for any $\alpha < \infty$, the updated policy $\pi_h^k$ in \eqref{eq:w1045} is ``conservatively'' greedy with respect to the Q-function $Q^{\pi^{k-1}, k-1}_h$ associated with the reward function $r^{k-1}$. In particular, the regularization effect of both $\pi^{k-1}_h$ and $\alpha$ in \eqref{eq:w1045} ensures that $\pi^k_h$ is not ``fully'' committed to perform well only with respect to $r^{k-1}$, just in case the subsequent adversarially chosen reward function $r^k$ significantly differs from $r^{k-1}$. In comparison, the ``fully'' greedy policy improvement
step, which is commonly adopted by the existing work on value-based reinforcement learning \citep{jaksch2010near, osband2014generalization, osband2016lower, azar2017minimax, dann2017unifying, strehl2006pac, jin2018q, jin2019provably, yang2019sample, yang2019reinforcement}, lacks such a notion of robustness. 
On the other hand, an intriguing question is whether being ``conservatively'' greedy is less sample-efficient than being ``fully'' greedy in the stationary setting, where the reward function is fixed across all the episodes. In fact, in the ideal setting where the Q-function $Q^{\pi^{k-1},k-1}_h$ associated with the reward function $r^{k-1}$ in \eqref{eq:w1045} is known, the ``fully'' greedy policy improvement
step with $\alpha\rightarrow \infty$ corresponds to one step of policy iteration \citep{sutton2018reinforcement}, which converges to the globally optimal policy $\pi^*$ within $K = H$ episodes and hence equivalently induces an $H^2$-regret. However, in the realistic setting, the Q-function $Q^{\pi^{k-1}, k-1}_h$ in \eqref{linearappr}-\eqref{eq:w1045} is replaced by the estimated Q-function $Q^{k-1}_h$ in Line 6 of Algorithm~\ref{ppoalgo}, which is obtained by the policy evaluation step defined in \eqref{eq:w1150}. As a result of the estimation uncertainty that arises from only observing finite historical data, it is indeed impossible to do better than the $\sqrt{T}$-regret even in the tabular setting \citep{jin2018q}, which is shown to be an information-theoretic lower bound. In the linear setting, OPPO attains such a lower bound in terms of the total number of steps $T = HK$. In other words, in the stationary setting, being ``conservatively'' greedy suffices to achieve sample-efficiency, which complements its advantages in terms of robustness in the more challenging setting with adversarially chosen reward functions.








%% file: sketch.tex
\section{Proof Sketch}\label{sketch}


\subsection{Regret Decomposition}\label{1114529}

For the simplicity of discussion, we define the model prediction error as 
\#\label{eq:w11260901}
 \iota^k_h=r^k_h+\mathbb{P}_{h}V^{k}_{h+1} - Q^{k}_h,
\#
which arises from estimating $\mathbb{P}_h V^k_{h+1}$ in the Bellman equation defined in \eqref{eq:w1204} (with $V^{\pi^k, k}_{h+1}$ replaced by $V^k_{h+1}$) based on only finite historical data. Also, we define the following filtration generated by the state-action sequence and reward functions.

\begin{definition}[Filtration]\label{def:w001}
 For any $(k,h)\in[K]\times[H]$, we define $\cF_{k,h,1}$ as the $\sigma$-algebra generated by the following state-action sequence and reward functions,
 \$
 \{(x^\tau_i, a^\tau_i)\}_{(\tau, i)\in [k-1] \times [H]} \cup \{r^\tau\}_{\tau\in [k]} \cup \{(x^k_i, a^k_i)\}_{i\in [h]}.
 \$
 For any $(k,h)\in[K]\times[H-1]$, we define $\cF_{k,h,2}$ as the $\sigma$-algebra generated by 
  \$
& \{(x^\tau_i, a^\tau_i)\}_{(\tau, i)\in [k-1] \times [H]} \cup \{r^\tau\}_{\tau\in [k]}   \cup \{(x^k_i, a^k_i)\}_{i\in [h]} \cup \{x^k_{h+1}\},
 \$
while for any $k\in[K]$ and $h=H$, we define $\cF_{k,h,2}$ as the $\sigma$-algebra generated by 
  \$
& \{(x^\tau_i, a^\tau_i)\}_{(\tau, i)\in [k] \times [H]} \cup \{r^\tau\}_{\tau\in [k+1]}.
 \$
By the above definitions, the $\sigma$-algebra sequence $\{\cF_{k,h,m}\}_{(k,h,m)\in[K]\times[H]\times[2]}$ is a filtration with respect to the timestep index
 \#\label{eq:w1155w}
 t(k,h,m)=(k-1)\cdot 2H+(h-1)\cdot2+m.
 \#
 In other words, for any $t(k,h,m) \le t(k',h',m')$, it holds that $\cF_{k,h,m}\subseteq \cF_{k',h',m'}$.
 
\end{definition}
By the definition of the $\sigma$-algebra $\cF_{k,h,m}$, for any $(k,h)\in[K]\times[H]$, the estimated value function $V^{k}_h$ and Q-function $Q^{k}_h$ are measurable to $\cF_{k,1,1}$, as they are obtained based on the $(k-1)$ historical trajectories and the reward function $r^k$ adversarially chosen by the environment at the beginning of the $k$-th episode, both of which are measurable to $\cF_{k,1,1}$. 

In the following lemma, we decompose the regret defined in \eqref{regret} into three terms. Recall that the globally optimal policy in hindsight  $\pi^*$ is defined in \eqref{eq:wpistar} and the model prediction error $\iota^k_h$ is defined in \eqref{eq:w11260901}.

\begin{lemma}[Regret Decomposition]\label{1005415}
It holds that
\# \label{1015243}
\text{Regret}(T) &= \sum_{k=1}^K \bigl(V^{\pi^*,k}_1(x^k_1) - V^{\pi^k,k}_1(x^k_1)\bigr) \notag\\
&=  \underbrace{\sum_{k=1}^K\sum_{h=1}^H \EE_{\pi^*} \bigl[ \la Q^{k}_h(x_h,\cdot), \pi^*_h(\cdot\,|\,x_h) - \pi^k_h(\cdot\,|\,x_h) \ra \bigr]}_{\dr (i)}  + \underbrace{ \cM_{K, H, 2}}_{\dr (ii)} \\
&\qquad+\underbrace{ \sum_{k=1}^K\sum_{h=1}^H\bigl( \EE_{\pi^*}[\iota^{k}_h(x_h,a_h)] - \iota^{k}_h(x^k_h,a^k_h)\bigr)}_{\dr (iii)},\notag
\#
which is independent of the linear setting in Assumption \ref{linearmdp}. 
Here $\{\cM_{k,h,m}\}_{(k,h,m)\in[K]\times[H]\times[2]}$ is a martingale adapted to the filtration $\{\cF_{k,h,m}\}_{(k,h,m)\in[K]\times[H]\times[2]}$, both with respect to the timestep index $t(k,h,m)$ defined in \eqref{eq:w1155w} of Definition \ref{def:w001}. 
\end{lemma}
\begin{proof}
See Appendix \ref{1013242} for a detailed proof.
\end{proof}

Lemma \ref{1005415} allows us to characterize the regret by upper bounding terms (i), (ii), and (iii) in \eqref{1015243}, respectively. In detail, term (i) corresponds to the right-hand side of \eqref{pdlm} in Lemma \ref{pdlm} with the Q-function $Q^{\pi^k,k}_h$ replaced by the estimated Q-function $Q^k_h$, which is obtained by the policy evaluation step defined in \eqref{eq:w1150}. In particular, as the updated policy $\pi^{k+1}_h$ is obtained by the policy improvement step in Line 6 of Algorithm \ref{ppoalgo} using $\pi^k_h$ and $Q^{k}_h$, term (i) can be upper bounded following a similar analysis to the discussion in Section \ref{secmd}, which is based on Lemmas \ref{pdlm} and \ref{mdlm} as well as \eqref{1015127}. Also, by the Azuma-Hoeffding inequality, term (ii) is a martingale that scales as $O(B_\cM\sqrt{T_{\cM}})$ with high probability, where $T_\cM$ is the total number of timesteps and $B_\cM$ is an upper bound of the martingale differences. More specifically, we prove that $T_{\cM}=2HK=2T$ and $B_{\cM}=2H$ in Appendix \ref{1008755}, which implies that term (ii) is $O(\sqrt{H^2T})$ with high probability.  Meanwhile, term (iii) corresponds to the model prediction error, which is characterized subsequently in Section \ref{secucb}. Note that the regret decomposition in \eqref{1015243} of Lemma \ref{1005415} is independent of the linear setting in Assumption \ref{linearmdp}, and therefore, applies to any forms of estimated Q-functions $Q^k_h$ in more general settings. In particular, as long as we can upper bound term (iii) in \eqref{1015243}, our regret analysis can be carried over even beyond the linear setting.

%

\subsection{Model Prediction Error}\label{secucb}
To upper bound term (iii) in \eqref{1015243} of Lemma \ref{1005415}, we characterize the model prediction error $\iota^k_h$ defined in \eqref{eq:w11260901} in the following lemma. Recall that the bonus function $\Gamma^{k}_h$ is defined in \eqref{1015509}.

\begin{lemma}[Upper Confidence Bound] \label{10061229}
Let $\lambda=1$ in \eqref{eq:w1150} and Line 12 of Algorithm \ref{ppoalgo}, and $\beta=C\sqrt{dH^2\cdot \lok}$ in \eqref{1015509} and Line 15 of Algorithm \ref{ppoalgo}, where $C>1$ is an absolute constant and $\zeta\in (0,1]$. Under Assumption \ref{linearmdp}, it holds with probability at least $1-\zeta/2$ that  
\$
 -2\Gamma^{k}_h(x,a) \le \iota^{k}_h(x,a) \le 0
 \$
 for any $(k,h)\in[K]\times[H]$ and $(x,a)\in\cS\times\cA$. 
\end{lemma}

\begin{proof}
See Appendix \ref{1013241} for a detailed proof. 
\end{proof}

Lemma \ref{10061229} demonstrates the key role of uncertainty quantification in achieving sample-efficiency. More specifically, due to the uncertainty that arises from only observing finite historical data, the model prediction error $\iota^{k}_h(x, a)$ can be possibly large for the state-action pairs $(x, a)$ that are less visited or even unseen. However, as is shown in Lemma \ref{10061229}, explicitly incorporating the bonus function $\Gamma^{k}_h$ into the estimated Q-function $Q^{k}_h$ ensures that $\iota^{k}_h(x, a) \leq 0$ with high probability for any $(k,h)\in[K]\times[H]$ and $(x,a)\in\cS\times\cA$. In other words, the estimated Q-function $Q^{k}_h$ is ``optimistic in the face of uncertainty'', as $\iota^{k}_h(x, a) \leq 0$ or equivalently
\#\label{eq:w450n}
 Q^{k}_h(x, a) \geq r^k_h(x, a)+ (\mathbb{P}_{h}V^{k}_{h+1})(x, a)
\# 
implies that
 $\EE_{\pi^*}[\iota^{k}_h(x_h,a_h)]$ in term (iii) of \eqref{1015243} is upper bounded by zero. Also, Lemma \ref{10061229} implies that $-\iota^{k}_h(x^k_h,a^k_h) \leq 2\Gamma^{k}_h(x^k_h,a^k_h)$ with high probability for any $(k,h)\in[K]\times[H]$. As a result, it only remains to upper bound the  cumulative sum $\sum_{k=1}^K\sum_{h=1}^H 2\Gamma^{k}_h(x^k_h,a^k_h)$ corresponding to term (iii) in \eqref{1015243}, which can be characterized by the elliptical potential lemma \citep{dani2008stochastic, rusmevichientong2010linearly, chu2011contextual, abbasi2011improved, jin2019provably}. See Appendix \ref{1008755} for a detailed proof.

To illustrate the intuition behind the model prediction error $\iota^k_h$ defined in \eqref{eq:w11260901}, we define the implicitly estimated transition dynamics as
\$
&\hat{\cP}_{k,h}(x'\,|\,x,a) =\psi(x,a,x')^\top(\Lambda^k_h)^{-1}\sum_{\tau=1}^{k-1} \phi^\tau_h(x_h^\tau,a_h^\tau)\cdot V^\tau_{h+1}(x^\tau_{h+1}),
\$
where $\Lambda^k_h$ is defined in \eqref{1015509}. Correspondingly, the policy evaluation step defined in \eqref{eq:w1150} takes the following equivalent form (ignoring the truncation step for the simplicity of discussion),
\#\label{eq:w418n}
Q^k_{h}\leftarrow r^k_h+\hat{\mathbb{P}}_{k,h}V^k_{h+1}+\Gamma^k_h.
\#
Here $\hat{\mathbb{P}}_{k,h}$ is the operator form of the implicitly estimated transition kernel $\hat{\cP}_{k,h}(\cdot\,|\,\cdot,\cdot)$, which is defined by 
\$
(\hat{\mathbb{P}}_{k,h}f) (x,a) = \int_\cS \hat{\cP}_{k,h}(x'\,|\,x,a) \cdot f(x')\bd x'	
\$
for any function $f:\cS\rightarrow \RR$. Correspondingly, by \eqref{eq:w11260901} and \eqref{eq:w418n} we have
\#\label{eq:w1245n}
\iota^k_h&=r^k_h+\mathbb{P}_hV^k_{h+1}-Q^k_h = (\mathbb{P}_h - \hat{\mathbb{P}}_{k,h}) V^{k}_{h+1}-\Gamma^k_h,
\#
where $\mathbb{P}_h - \hat{\mathbb{P}}_{k,h}$ is the error that arises from implicitly estimating the transition dynamics based on only finite historical data. Such a model estimation error enters the regret in \eqref{1015243} of Lemma \ref{1005415} only through the model prediction error $(\mathbb{P}_h - \hat{\mathbb{P}}_{k,h}) V^{k}_{h+1}$, which allows us to bypass explicitly estimating the transition dynamics, and instead,  employ the estimated Q-function $Q^k_h$ obtained by the policy evaluation step defined in \eqref{eq:w418n}. As is shown in Appendix \ref{1013241}, the bonus function $\Gamma^k_h$ upper bounds $(\mathbb{P}_h - \hat{\mathbb{P}}_{k,h}) V^{k}_{h+1}$ in \eqref{eq:w1245n} with high probability for any $(k,h)\in[K]\times[H]$ and $(x, a) \in \cS \times \cA$, which then ensures the optimism of the estimated Q-function $Q^{k}_h$ in the sense of \eqref{eq:w450n}.

%% file: appendix.tex
\section{Proofs of Lemmas in Section \ref{results}}
\subsection{Proof of Lemma \ref{pdlm}} \label{pdlmpf}
\begin{proof}
In this section, we focus on the $k$-th episode and omit the episode index $k$ for notational simplicity. For any $h\in[H]$ and policy $\pi\in\Delta(\cA\,|\,\cS,H)$, we define the Bellman evaluation operator $\mathbb{T}_{h,\pi}$ by
\#\label{eq:w635n}
(\mathbb{T}_{h,\pi} V)(x)&=\EE[ r_h(x,a) + V(x') \,|\, a\sim \pi_{h}(\cdot\,|\,x) ,\, x'\sim \cP_h(\cdot\,|\,x,a)] \notag\\
&= \la (r_h + \mathbb{P}_h V)(x,\cdot), \pi_h(\cdot\,|\,x)  \ra
\#
for any function $V:\cS\rightarrow \RR$. By the definition of the value function $V^{\pi}_h$ in \eqref{eq:vfunc}, we have
\#\label{1009258}
V^{\pi}_h = \prod_{i=h}^H \mathbb{T}_{i,\pi} \mathbf{0}
\#
for any $h\in[H]$, where $\mathbf{0}$ is a zero function on $\cS$. Here $\prod_{i=h}^H \mathbb{T}_{i,\pi}$ denotes the sequential composition of the Bellman evaluation operators $\mathbb{T}_{i,\pi}$. Thus, for any policies $\pi', \pi\in\Delta(\cA\,|\,\cS,H)$, it holds that
\#\label{1009317}
V^{\pi'}_1- V^{\pi}_1 &=  \prod_{h=1}^H \mathbb{T}_{h,\pi'} \mathbf{0} -  \prod_{h=1}^H \mathbb{T}_{h,\pi} \mathbf{0} \notag \\
&= \prod_{h=1}^H \mathbb{T}_{h,\pi'} \mathbf{0}- \sum_{h=1}^{H-1} \Bigl(\prod_{i=1}^h \mathbb{T}_{i,\pi'} \prod_{i=h+1}^H \mathbb{T}_{i,\pi} \mathbf{0}-\prod_{i=1}^h \mathbb{T}_{i,\pi'} \prod_{i=h+1}^H \mathbb{T}_{i,\pi}\mathbf{0}  \Bigr)  -  \prod_{h=1}^H \mathbb{T}_{h,\pi} \mathbf{0} \notag   \\
&=\sum_{h=H}^1 \Bigl(\prod_{i=1}^h \mathbb{T}_{i,\pi'} \prod_{i=h+1}^H \mathbb{T}_{i,\pi}\mathbf{0}-\prod_{i=1}^{h-1} \mathbb{T}_{i,\pi'} \prod_{i=h}^H \mathbb{T}_{i,\pi}\mathbf{0}\Bigr).
\#
 Meanwhile, by \eqref{1009258} we have that, on the right-hand side of \eqref{1009317},
\#\label{1009318}
&\prod_{i=1}^h \mathbb{T}_{i,\pi'} \prod_{i=h+1}^H \mathbb{T}_{i,\pi}\mathbf{0} -\prod_{i=1}^{h-1} \mathbb{T}_{i,\pi'} \prod_{i=h}^H \mathbb{T}_{i,\pi}\mathbf{0} \notag  \\
&\quad= \prod_{i=1}^{h-1} \mathbb{T}_{i,\pi'} (\mathbb{T}_{h,\pi'}-\mathbb{T}_{h,\pi} )\prod_{i=h+1}^H \mathbb{T}_{i,\pi} \mathbf{0}
= \prod_{i=1}^{h-1} \mathbb{T}_{i,\pi'} (\mathbb{T}_{h,\pi'}-\mathbb{T}_{h,\pi} ) V^{\pi}_{h+1}.
\#
By the definition of the Bellman evaluation operator $\mathbb{T}_{h,\pi}$ in \eqref{eq:w635n}, we have
\#\label{1009319}
&(\mathbb{T}_{h,\pi'}-\mathbb{T}_{h,\pi} ) V^{\pi}_{h+1} = \la r_h+\mathbb{P}_hV^\pi_{h+1}, \pi'_h-\pi_h \ra_{\cA}= \la Q^\pi_{h}, \pi'_h-\pi_h \ra_{\cA},
\#
where the last equality follows from \eqref{eq:w1204}. Combining \eqref{1009317}, \eqref{1009318}, \eqref{1009319}, and the linearity of the Bellman evaluation operator defined in \eqref{eq:w635n}, we obtain
\$
V^{\pi'}_1(x_1)- V^{\pi}_1(x_1) &=\sum_{h=1}^H \Bigl(\prod_{i=1}^{h-1} \mathbb{T}_{i,\pi'}  \la Q^\pi_{h}, \pi'_h-\pi_h \ra_{\cA} \Bigr)(x_1) \\
&= \EE_{\pi'}\Bigl[\sum_{h=1}^H \la Q^{\pi}_h(x_h,\cdot), \pi'_h(\cdot\,|\,x_h) - \pi_h(\cdot\,|\,x_h)  \ra\,\Big|\, x_1 \Bigr],
\$ 
which concludes the proof of Lemma \ref{pdlm}.
\end{proof}

\subsection{Proof of Lemma \ref{mdlm}} \label{mdlmpf}
\begin{proof}
For any function $g:\cA\rightarrow\RR$ and distributions $p,p',p^* \in\Delta(\cA)$ that satisfy 
\$
p'(\cdot)\propto p(\cdot)\cdot \exp\bigl(\alpha\cdot g(\cdot)\bigr),
\$ 
we have 
\#\label{1009403}
\alpha\cdot \la g, p^* - p' \ra &=  \la z+\log(p'/p), p^*-p' \ra \notag\\
&=\la z, p^*-p' \ra + \la \log(p^*/p), p^* \ra + \la \log(p'/p^*), p^* \ra  + \la \log(p'/p), -p' \ra \notag\\
&= \dkl(p^*\,\|\,p) - \dkl(p^*\,\|\,p') - \dkl(p'\,\|\,p).
\#
Here $z: \cA \rightarrow \RR$ is a constant function defined by
\$
z(a) = \log\Bigl(\sum_{a'\in\cA}p(a')\cdot\exp\bigl(\alpha\cdot g(a')\bigr)\Bigr),
\$
which implies that $\la z, p^*-p' \ra = 0$ in \eqref{1009403} as $p',p^* \in\Delta(\cA)$. Moreover, by \eqref{1009403} we have
\#\label{1009359}
\alpha \cdot \la Q(x,\cdot), {p}^*(\cdot) - {p}(\cdot) \ra&= \alpha \cdot \la Q(x,\cdot), {p}^*(\cdot) - {p}'(\cdot) \ra - \alpha \cdot \la Q(x,\cdot), {p}(\cdot) - {p}'(\cdot) \ra \notag \\
&  \le \dkl\bigl(p^*(\cdot)\,\big\|\,p(\cdot)\bigr) - \dkl\bigl(p^*(\cdot)\,\big\|\,p'(\cdot)\bigr) - \dkl\bigl(p'(\cdot)\,\big\|\,p(\cdot)\bigr) \\
&\qquad+ \alpha \cdot \|Q(x,\cdot)\|_{\infty}\cdot \| {p}(\cdot) - {p}'(\cdot)\|_1\notag
\#
for any state $x \in \cS$. Meanwhile, by Pinsker's inequality, it holds that 
\#\label{eq:wpinsker}
\dkl({p}' \,\|\, {p})\ge \|{p}-{p}'\|^2_1/2.
\#
Combining \eqref{1009359}, \eqref{eq:wpinsker}, and the fact that $\|Q(x,\cdot)\|_{\infty}\le H$ for any state $x \in \cS$, we obtain
\$
&\alpha \cdot \la Q(x,\cdot), {p}^*(\cdot) - {p}(\cdot) \ra \\
& \quad \le \dkl\bigl(p^*(\cdot)\,\big\|\,p(\cdot)\bigr) - \dkl\bigl(p^*(\cdot)\,\big\|\,p'(\cdot)\bigr) - \| {p}(\cdot) - {p}'(\cdot)\|_1^2/2 + \alpha H \cdot \| {p}(\cdot) - {p}'(\cdot)\|_1 \\
& \quad \le \dkl\bigl(p^*(\cdot)\,\big\|\,p(\cdot)\bigr) - \dkl\bigl(p^*(\cdot)\,\big\|\,p'(\cdot)\bigr)  + \alpha^2 H^2/2,
\$
which concludes the proof of Lemma \ref{mdlm}.
\end{proof}

\section{Proofs of Lemmas in Section \ref{sketch} }
For notational simplicity, we define the operators $\mathbb{J}_h$ and $\mathbb{J}_{k,h}$ respectively by
\#\label{eq:w256a}
(\mathbb{J}_h f)(x) = \la f(x,\cdot), \pi^*_h(\cdot\,|\,x) \ra, \quad
(\mathbb{J}_{k,h} f)(x) = \la f(x,\cdot), \pi^k_h(\cdot\,|\,x) \ra
\#
for any $(k,h)\in[K]\times[H]$ and function $f: \cS\times\cA\rightarrow \RR$. Also, we define
\#\label{eq:w256b}
\xi^k_h(x) = (\mathbb{J}_{h} Q^{k}_h)(x) - (\mathbb{J}_{k,h} Q^{k}_h)(x) = \la Q^{k}_h(x,\cdot), \pi^*_h(\cdot\,|\,x) - \pi^k_h(\cdot\,|\,x) \ra
\#
for any $(k,h)\in[K]\times[H]$ and state $x\in\cS$.
\subsection{Proof of Lemma \ref{1005415}} \label{1013242}
\begin{proof}

We decompose the instantaneous regret at the $k$-th episode into the following two terms,
\#\label{1015212}
V^{\pi^*,k}_1(x_1^k) - V^{\pi^k,k}_1(x_1^k) = \underbrace{V^{\pi^*,k}_1(x_1^k) - V^{k}_1(x_1^k)}_{\dr (i)} + \underbrace{V^{k}_1(x_1^k) - V^{\pi^k,k}_1(x_1^k)}_{\dr (ii)}. 
\#

\vspace{4pt}
\noindent
{\bf Term (i).} By the definitions of the value function $V^{\pi^*, k}_h$ in \eqref{eq:w1204}, the estimated value function $V^{k}_h$ in \eqref{eq:w230a}, the operators $\mathbb{J}_h$ and $\mathbb{J}_{k,h}$ in \eqref{eq:w256a}, and $\xi^k_h$ in \eqref{eq:w256b}, we have
\#\label{1005609}
V^{\pi^*,k}_h - V^{k}_h &=  \mathbb{J}_h Q^{\pi^*,k}_h - \mathbb{J}_{k,h} Q^{k}_h \notag\\
&= \mathbb{J}_h ( Q^{\pi^*,k}_h - Q^{k}_h) + (\mathbb{J}_h - \mathbb{J}_{k,h} )Q^{k}_h
 = \mathbb{J}_h ( Q^{\pi^*,k}_h - Q^{k}_h) + \xi^k_h
\#
for any $(k,h)\in[K]\times[H]$. Meanwhile, by the definition of the model prediction error, that is, $ \iota^k_h=r_h^k+\mathbb{P}_{h}V^{k}_{h+1} - Q^{k}_h$, we have that, on the right-hand side of \eqref{1005609},
\$
Q^{\pi^*,k}_h &= r^k_h + \mathbb{P}_h V^{\pi^*,k}_{h+1}, \quad
Q^{k}_h = r^k_h + \mathbb{P}_h V^{k}_{h+1} - \iota^{k}_h,
\$ 
which implies 
\#\label{1005610}
Q^{\pi^*,k}_h - Q^{k}_h = \mathbb{P}_h (V^{\pi^*,k}_{h+1} - V^{k}_{h+1}) + \iota^{k}_h. 
\#
Combining \eqref{1005609} and \eqref{1005610}, we obtain
\#\label{1005611}
V^{\pi^*,k}_h - V^{k}_h =
\mathbb{J}_h\mathbb{P}_h (V^{\pi^*,k}_{h+1} - V^{k}_{h+1}) + \mathbb{J}_h \iota^{k}_h + \xi^k_h
\#
for any $(k,h)\in[K]\times[H]$. For any $k \in [K]$, recursively expanding \eqref{1005611} across $h \in [H]$ yields
\$
&V^{\pi^*,k}_1 - V^{k}_1  =
\Bigl(\prod_{h=1}^H \mathbb{J}_h\mathbb{P}_h \Bigr) (V^{\pi^*,k}_{H+1} - V^{k}_{H+1})
+ \sum_{h=1}^H \Bigl(\prod_{i=1}^{h-1} \mathbb{J}_i\mathbb{P}_i \Bigr)
\mathbb{J}_h \iota^{k}_h +
\sum_{h=1}^H \Bigl(\prod_{i=1}^{h-1} \mathbb{J}_i\mathbb{P}_i \Bigr) \xi^k_h,
\$
where $V^{\pi^*,k}_{H+1} = V^{k}_{H+1}=\zero$. Therefore, we obtain
\$
&V^{\pi^*,k}_1 - V^{k}_1 =\sum_{h=1}^H \Bigl(\prod_{i=1}^{h-1} \mathbb{J}_i\mathbb{P}_i \Bigr)
\mathbb{J}_h \iota^{k}_h +
\sum_{h=1}^H \Bigl(\prod_{i=1}^{h-1} \mathbb{J}_i\mathbb{P}_i \Bigr) \xi^k_h.
\$
By the definitions of $\mathbb{P}_h$ in \eqref{eq:w256c}, $\mathbb{J}_h$ in \eqref{eq:w256a}, and $\xi^k_h$ in \eqref{eq:w256b}, we further obtain 
\#\label{1005825}
&V^{\pi^*,k}_1(x_1^k) - V^{k}_1(x_1^k) \\
&\quad=\sum_{h=1}^H \EE_{\pi^*}[\iota^{k}_h(x_h,a_h)\,|\, x_1 = x_1^k] +
\sum_{h=1}^H \EE_{\pi^*} \bigl[\la Q^{k}_h(x_h,\cdot), \pi^*_h(\cdot\,|\,x_h) - \pi^k_h(\cdot\,|\,x_h) \ra \,\big|\, x_1 = x_1^k\bigr]\notag
\#
for any $k \in [K]$.

\vspace{4pt}
\noindent
{\bf Term (ii).} By the definitions of the value function $V^{\pi^k, k}_h$ in \eqref{eq:w1204}, the estimated value function $V^{k}_h$ in \eqref{eq:w230a}, and the operator $\mathbb{J}_{k,h}$ in \eqref{eq:w256a}, we have
\#\label{1005720}
&V^{k}_h(x^k_h) - V^{\pi^k,k}_h(x^k_h) =\bigl(\mathbb{J}_{k,h} (Q^{k}_h-Q^{\pi^k,k}_h)\bigr)(x^k_h) + \iota^{k}_h(x^k_h,a^k_h) - \iota^{k}_h(x^k_h,a^k_h)
\#
for any $(k,h)\in[K]\times[H]$. By the definition of the model prediction error $\iota^{k}_h$ in \eqref{eq:w11260901}, we have
\#\label{1005721}
\iota^{k}_h(x^k_h,a^k_h)&= r^k_h(x^k_h,a^k_h) + (\mathbb{P}_h V^{k}_{h+1})(x^k_h,a^k_h) - Q^{k}_h(x^k_h,a^k_h)  \notag\\
&=\bigl( r^k_h(x^k_h,a^k_h) + (\mathbb{P}_h V^{k}_{h+1})(x^k_h,a^k_h) - Q^{\pi^k,k}_h(x^k_h,a^k_h) \bigr) + \bigl( Q^{\pi^k,k}_h(x^k_h,a^k_h) - Q^{k}_h(x^k_h,a^k_h) \bigr) \notag\\
&=\bigl(\mathbb{P}_h (V^{k}_{h+1}-V^{\pi^k,k}_{h+1})\bigr)(x^k_h,a^k_h) + ( Q^{\pi^k,k}_h- Q^{k}_h)(x^k_h,a^k_h),
\#
where the last equality follows from \eqref{eq:w1204}. 
Plugging \eqref{1005721} into \eqref{1005720}, we obtain
\#\label{1005805a}
V^{k}_h(x^k_h) - V^{\pi^k,k}_h(x^k_h) &=\bigl(\mathbb{J}_{k,h} (Q^{k}_h-Q^{\pi^k,k}_h)\bigr)(x^k_h) + ( Q^{\pi^k,k}_h- Q^{k}_h)(x^k_h,a^k_h)  \\
&\qquad + \bigl(\mathbb{P}_h (V^{k}_{h+1}-V^{\pi^k,k}_{h+1})\bigr)(x^k_h,a^k_h) -  \iota^{k}_h(x^k_h,a^k_h),\notag
\#
which implies
\#\label{1005805}
V^{k}_h(x^k_h) - V^{\pi^k,k}_h(x^k_h)& = \underbrace{\bigl(\mathbb{J}_{k,h} (Q^{k}_h-Q^{\pi^k,k}_h)\bigr)(x^k_h) - ( Q^{k}_h-Q^{\pi^k,k}_h)(x^k_h,a^k_h)}_{\displaystyle D_{k,h,1}} \\
&\qquad + \underbrace{ \bigl(\mathbb{P}_h (V^{k}_{h+1}-V^{\pi^k,k}_{h+1})\bigr)(x^k_h,a^k_h) - (V^{k}_{h+1}-V^{\pi^k,k}_{h+1})(x^k_{h+1})}_{\displaystyle D_{k,{h},2}}\notag\\
&\qquad + (V^{k}_{h+1}-V^{\pi^k,k}_{h+1})(x^k_{h+1})-  \iota^{k}_h(x^k_h,a^k_h)\notag
\#
for any $(k,h)\in[K]\times[H]$. For any $k \in [K]$, recursively expanding \eqref{1005805} across $h \in [H]$ yields
\$
&V^{k}_1(x^k_1) - V^{\pi^k,k}_1(x^k_1) \\
&\quad= 
V^{k}_{H+1}(x^k_{H+1}) - V^{\pi^k,k}_{H+1}(x^k_{H+1}) - \sum_{h=1}^H \iota^{k}_h(x^k_h,a^k_h) + \sum_{h=1}^H (D_{k,h,1}+D_{k,h,2}),
\$
where $V^{k}_{H+1}(x^k_{H+1}) = V^{\pi^k,k}_{H+1}(x^k_{H+1})=0$. Therefore, we obtain
\#\label{1005810}
&V^{k}_1(x^k_1) - V^{\pi^k,k}_1(x^k_1) = 
 - \sum_{h=1}^H \iota^{k}_h(x^k_h,a^k_h) + \sum_{h=1}^H (D_{k,h,1}+D_{k,h,2}).
\#
By Definition \ref{def:w001} and the definitions of $D_{k,h,1}$ and $D_{k,h,2}$ in \eqref{1005805}, we have
\#\label{eq:w560a}
D_{k,h,1}\in\cF_{k,h,1},~~D_{k,h,2}\in\cF_{k,h,2},~~\EE[D_{k,h,1}\,|\,\cF_{k,h-1,2}]=0,~~\EE[D_{k,h,2}\,|\,\cF_{k,h,1}]=0
\#
for any $(k,h)\in[K]\times[H]$. Here we have that $\cF_{k,0,2}=\cF_{k-1,H,2}$ for any $k\ge2$, as \eqref{eq:w1155w} of Definition \ref{def:w001} implies
\$
t(k, 0, 2) = t(k-1, H, 2) = (k-1)\cdot 2H.
\$ 
Also, we define $\cF_{1,0,2}$ to be empty. Thus, \eqref{eq:w560a} allows us to define the  martingale
\#\label{eq:w445ak}
\cM_{k,h,m} &= \sum_{\tau=1}^{k-1}\sum_{i=1}^H (D_{\tau,i,1}+D_{\tau,i,2}) + \sum_{i=1}^{h-1} (D_{k,i,1}+D_{k,i,2}) + \sum_{\ell=1}^{m} D_{k,h,\ell} \notag\\
&= \sum_{\substack{(\tau,i,\ell)\in[K]\times[H]\times[2],\\ t(\tau,i,\ell)\le t(k,h,m)}} D_{\tau,i,\ell}
\# 
with respect to the timestep index $t(k,h,m)$ defined in \eqref{eq:w1155w} of Definition \ref{def:w001}. Such a~martingale is adapted to the filtration $\{\cF_{k,h,m}\}_{(k,h,m)\in[K]\times[H]\times[2]}$. In particular, we have that, on the right-hand side of \eqref{1005810},
\#\label{1005844}
\sum_{k=1}^K\sum_{h=1}^{H} (D_{k,h,1}+D_{k,h,2}) = \cM_{K,H, 2}.
\#

Combining \eqref{1015212}, \eqref{1005825}, \eqref{1005810}, and \eqref{1005844}, we obtain 
\$
\sum_{k=1}^K \bigl( V^{\pi^*,k}_1(x_1^k) - V^{\pi^k,k}_1(x_1^k) \bigr) &= \sum_{k=1}^K \sum_{h=1}^H \EE_{\pi^*}[\iota^{k}_h(x_h,a_h)\,|\, x_1 = x_1^k] \\
&\qquad +
\sum_{k=1}^K \sum_{h=1}^H \EE_{\pi^*} \bigl[\la Q^{k}_h(x_h,\cdot), \pi^*_h(\cdot\,|\,x_h) - \pi^k_h(\cdot\,|\,x_h) \ra\,\big|\, x_1 = x_1^k\bigr] \\
&\qquad - \sum_{k=1}^K\sum_{h=1}^H \iota^{k}_h(x^k_h,a^k_h) + \cM_{K,H, 2},
\$
which concludes the proof of Lemma \ref{1005415}.
\end{proof}

\subsection{Proof of Lemma \ref{10061229}} \label{1013241}

\begin{proof}
Recall that $\phi_{h}^k$ defined in \eqref{eq:w230a} takes the following form,
\$
\phi_h^k (x,a) = \int_{\cS} \psi (x,a, x') \cdot V_{h+1}^k (x') \ud x'
\$ 
for any $(k,h) \in [ K ] \times [H]$ and $(x,a) \in \cS \times \cA$. 
Also, recall that 
the estimated Q-function $Q^{k}_h$ obtained by the policy evaluation step defined in \eqref{eq:w1150} takes the following form,
\#\label{1006221}
&Q^{k}_h(x,a) = \min\{ r^k_h(x,a) + \phi^k_h(x,a)^\top w^k_h + \Gamma^k_h(x,a), H-h+1\}^+,\\
&\text{where~~}w^{k}_h   = (\Lambda^k_h)^{-1} \Bigl( \sum_{\tau=1}^{k-1} \phi^\tau_h(x_h^\tau,a_h^\tau)\cdot V_{h+1}^\tau(x_{h+1}^\tau) \Bigr)\notag
\#
for any $(k,h) \in [ K ] \times [H]$ and $(x,a) \in \cS \times \cA$. Here $\Gamma^{k}_h$ and $\Lambda_h^k$ are  defined in \eqref{1015509}. Meanwhile, by Assumption \ref{linearmdp} we have
\#\label{1007445}
(\mathbb{P}_h V^{k}_{h+1})(x,a)&= \int_{\cS} \psi(x,a,x')^\top \theta_h \cdot V^k_{h+1}(x')\ud x '  \notag\\
&= \phi^k_h(x,a)^\top \theta_h= \phi^k_h(x,a)^\top(\Lambda^k_h)^{-1}\Lambda^k_h \theta_h
\#
for any $(k,h) \in [ K ] \times [H]$ and $(x,a) \in \cS \times \cA$. Plugging the definition of $\Lambda^{k}_h$ in \eqref{1015509} into \eqref{1007445}, we obtain
\#\label{1006222}
(\mathbb{P}_h V^{k}_{h+1})(x,a) &= \phi^k_h(x,a)^\top (\Lambda^{k}_h)^{-1} \Bigl(\sum_{\tau=1}^{k-1}\phi^\tau_h(x_h^\tau,a_h^\tau)\phi^\tau_h(x_h^\tau,a_h^\tau)^\top\theta_h + \lambda\cdot \theta_h \Bigr)\notag\\
&= \phi^k_h(x,a)^\top (\Lambda^{k}_h)^{-1} \Bigl(\sum_{\tau=1}^{k-1}\phi^\tau_h(x_h^\tau,a_h^\tau)\cdot (\mathbb{P}_hV^{\tau}_{h+1})(x_h^\tau,a_h^\tau) + \lambda\cdot \theta_h \Bigr)
\#
for any $(k,h) \in [ K ] \times [H]$ and $(x,a) \in \cS \times \cA$. Here the second equality follows from \eqref{1007445} with $V^{k}_{h+1}$ replaced by $V_{h+1}^\tau$ for any $\tau\in [k-1]$. Combining \eqref{1006221} and \eqref{1006222}, we obtain
\#\label{1006646}
&\phi^k_h(x,a)^\top w^{k}_h - (\mathbb{P}_h V^{k}_{h+1})(x,a) \\
&\quad=\underbrace{\phi^k_h(x,a)^\top(\Lambda^{k}_h)^{-1} \Bigl(\sum_{\tau=1}^{k-1} \phi^\tau_h(x_h^\tau,a_h^\tau)
\cdot \bigl(V^{\tau}_{h+1}(x_{h+1}^\tau) - (\mathbb{P}_hV^{\tau}_{h+1})(x_h^\tau,a_h^\tau)\bigr)\Bigr)}_{\dr (i)}  - \underbrace{\lambda\cdot
\phi^k_h(x,a)^\top (\Lambda^{k}_h)^{-1} \theta_h }_{\dr (ii)}\notag
\#
for any $(k,h) \in [ K ] \times [H]$ and $(x,a) \in \cS \times \cA$.

\vspace{4pt}
\noindent
{\bf Term (i).} As is defined in \eqref{1015509}, $(\Lambda^{k}_h)^{-1}$ is a positive-definite matrix. By the Cauchy-Schwarz inequality, the absolute value of term (i) is upper bounded as
\#\label{1006643}
|{\dr (i)}|
&\le \sqrt{\phi^k_h(x,a)^\top(\Lambda^{k}_h)^{-1}\phi^k_h(x,a)} \cdot \Bigl\| \sum_{\tau=1}^{k-1} \phi^\tau_h(x_h^\tau,a_h^\tau) \cdot \bigl( V^{\tau}_{h+1}(x^\tau_{h+1}) - (\mathbb{P}_hV^{\tau}_{h+1})(x^\tau_{h}, a^\tau_h) \bigr) \Bigr\|_{(\Lambda^{k}_h)^{-1}}
\#
for any $(k,h) \in [ K ] \times [H]$ and $(x,a) \in \cS \times \cA$. Under the event $\mathcal{E}$ defined in \eqref{event} of Lemma \ref{eventlm}, which happens with probability at least $1-\zeta/2$,
it holds that 
\#\label{200703356}
|{\dr (i)}| \le C'' \sqrt{dH^2\cdot\lok}\cdot\sqrt{\phi^k_h(x,a)^\top(\Lambda^{k}_h)^{-1}\phi^k_h(x,a)}
\#
for any $(k,h) \in [ K ] \times [H]$ and $(x,a) \in \cS \times \cA$. Here $C''>0$ is an absolute constant and $\zeta\in (0,1]$.

\vspace{4pt}
\noindent
{\bf Term (ii).} Similar to \eqref{1006643}, the absolute value of term (ii) is upper bounded as 
\#\label{1006644}
|{\dr (ii)}|&\le \lambda\cdot \sqrt{\phi^k_h(x,a)^\top(\Lambda^{k}_h)^{-1}\phi^k_h(x,a)} \cdot \| \theta_h \|_{(\Lambda^{k}_h)^{-1}} \notag\\
& \le \sqrt{\lambda}\cdot \sqrt{\phi^k_h(x,a)^\top(\Lambda^{k}_h)^{-1}\phi^k_h(x,a)} \cdot \|  \theta_h \|_2 \le \sqrt{\lambda d}\cdot \sqrt{\phi^k_h(x,a)^\top(\Lambda^{k}_h)^{-1}\phi^k_h(x,a)}
\#
for any $(k,h) \in [ K ] \times [H]$ and $(x,a) \in \cS \times \cA$. Here the first inequality follows from the Cauchy-Schwarz inequality, 
the second  inequality follows from the fact that $\Lambda^{k}_h\succeq \lambda\cdot{I}$, and the last inequality follows from Assumption \ref{linearmdp}, which assumes that $\|  \theta_h \|_2 \leq \sqrt{d}$.

Combining \eqref{1006646}, \eqref{200703356}, \eqref{1006644}, and the fact that $\lambda = 1$, we obtain
\#\label{10081252}
&| \phi^k_h(x,a)^\top w^{k}_h - (\mathbb{P}_h V^{k}_{h+1}) (x,a)| \notag\\
&\quad \le C \sqrt{dH^2 \cdot \lok}\cdot\sqrt{\phi^k_h(x,a)^\top(\Lambda^{k}_h)^{-1}\phi^k_h(x,a)} 
\#
for any $(k,h) \in [ K ] \times [H]$ and $(x,a) \in \cS \times \cA$ under the event $\mathcal{E}$ defined in \eqref{event} of Lemma \ref{eventlm}. Here $C> 1$ is an absolute constant. Setting 
\$
\beta = C \sqrt{dH^2 \cdot \lok}
\$
in the bonus function $\Gamma^{k}_h$ defined in \eqref{1015509}, by \eqref{10081252} we obtain 
\#\label{eq:w1048k}
| \phi^k_h(x,a)^\top w^{k}_h - (\mathbb{P}_h V^{k}_{h+1})(x,a)| \le \Gamma^k_h(x,a)
\#
for any $(k,h) \in [ K ] \times [H]$ and $(x,a) \in \cS \times \cA$ under $\mathcal{E}$. As \eqref{eq:w1150} implies that $(\mathbb{P}_h V^{k}_{h+1})(x, a)\ge0$, by \eqref{eq:w1048k} we have 
 \#\label{eq:w1055k}
 \phi^k_h(x,a)^\top w^{k}_h + \Gamma^{k}_h(x,a)\ge0
 \# 
 for any $(k,h) \in [ K ] \times [H]$ and $(x,a) \in \cS \times \cA$ under $\mathcal{E}$. Hence, for the model prediction error $\iota^{k}_h$ defined in \eqref{eq:w11260901}, by \eqref{1006221}, \eqref{eq:w1048k}, and \eqref{eq:w1055k} we have
\#\label{eq:w1100k}
-\iota^{k}_h(x,a)&=Q^{k}_h(x,a) - (r^k_h + \mathbb{P}_h V^{k}_{h+1} )(x,a)\notag\\
&\le
r^k_h(x,a)+\phi^k_h(x,a)^\top w^{k}_h + \Gamma^{k}_h(x,a) - (r^k_h + \mathbb{P}_h V^{k}_{h+1} )(x,a) \le  2\Gamma^{k}_h(x,a)
\#
for any $(k,h) \in [ K ] \times [H]$ and $(x,a) \in \cS \times \cA$ under $\mathcal{E}$. Meanwhile, as \eqref{eq:w1150} implies that $(\mathbb{P}_h V^{k}_{h+1})(x, a)\leq H-h$ and hence 
\$
(r^k_h + \mathbb{P}_h V^{k}_{h+1})(x,a)\leq H-h+1,
\$
by \eqref{eq:w11260901}, \eqref{1006221}, and \eqref{eq:w1048k} we have
\#\label{eq:w1101k}
\iota^{k}_h(x,a)&=(r^k_h + \mathbb{P}_h V^{k}_{h+1} )(x,a)-Q^{k}_h(x,a)\notag\\
&\le
(r^k_h+\mathbb{P}_h V^{k}_{h+1}) (x,a) - 
 \min\{r^k_h(x,a) +\phi^k_h(x,a)^\top w^{k}_h + \Gamma^{k}_h(x,a), H-h+1\} \notag\\
&= \max\{ (\mathbb{P}_h V^{k}_{h+1})(x,a)-\phi^k_h(x,a)^\top w^{k}_h - \Gamma^{k}_h(x,a), (r^k_h+\mathbb{P}_h V^{k}_{h+1}) (x,a)-(H-h+1)\}\notag\\
&\le0
\#
for any $(k,h) \in [ K ] \times [H]$ and $(x,a) \in \cS \times \cA$ under $\mathcal{E}$. Thus, combining \eqref{eq:w1100k}, \eqref{eq:w1101k}, and Lemma \ref{eventlm}, which ensures that $\mathcal{E}$ happens with probability at least $1-\zeta/2$, we conclude the proof of Lemma \ref{10061229}.
\end{proof}

\section{Proof of Theorem \ref{10061231}}\label{1008755}
\begin{proof}
We upper bound terms (i)-(iii) in \eqref{1015243} of Lemma \ref{1005415} respectively, that is, 
\$
\text{Regret}(T) &= \sum_{k=1}^K \bigl(V^{\pi^*,k}_1(x^k_1) - V^{\pi^k,k}_1(x^k_1)\bigr) \notag\\
&=  \underbrace{\sum_{k=1}^K\sum_{h=1}^H \EE_{\pi^*} \bigl[ \la Q^{k}_h(x_h,\cdot), \pi^*_h(\cdot\,|\,x_h) - \pi^k_h(\cdot\,|\,x_h) \ra \,\big|\, x_1 = x^k_1 \bigr]}_{\dr (i)} + \underbrace{ \cM_{K, H, 2}}_{\dr (ii)}  \\
& \qquad +\underbrace{ \sum_{k=1}^K\sum_{h=1}^H\bigl( \EE_{\pi^*}[\iota^{k}_h(x_h,a_h)\,|\, x_1 = x^k_1] - \iota^{k}_h(x^k_h,a^k_h)\bigr)}_{\dr (iii)}.\notag
\$

\vspace{4pt}
\noindent
\textbf{Term (i).} By Lemma \ref{mdlm} and the policy improvement step in Line 6 of Algorithm \ref{ppoalgo}, we have
\#\label{eq:w361ok}
&\sum_{k=1}^K\sum_{h=1}^H \EE_{\pi^*} \bigl[ \la Q^{k}_h(x_h,\cdot), \pi^*_h(\cdot\,|\,x_h) - \pi^k_h(\cdot\,|\,x_h) \ra \,\big|\, x_1 = x^k_1 \bigr] \notag\\
& \le 
\sum_{k=1}^K\sum_{h=1}^H \Bigl( \alpha H^2/2 + \alpha^{-1} \cdot \EE_{\pi^*}\bigl[\dkl\bigl( \pi^*_h(\cdot\,|\,x_h) \,\big\|\, \pi^k_h(\cdot\,|\,x_h) \bigr) -
\dkl\bigl( \pi^*_h(\cdot\,|\,x_h) \,\big\|\, \pi^{k+1}_h(\cdot\,|\,x_h) \bigr) \,\big|\, x_1 = x^k_1\bigr] \Bigr)\notag\\
& \le
\alpha H^3K/2 + \alpha^{-1} \cdot \sum_{h=1}^H \EE_{\pi^*} \bigl[
\dkl\bigl( \pi^*_h(\cdot\,|\,x_h) \,\big\|\, \pi^1_h(\cdot\,|\,x_h) \bigr) 
\,\big|\, x_1 = x^k_1 \bigr]\notag\\
 & \le \alpha H^3K/2+\alpha^{-1} H\cdot \log{|\cA|}.
\#
Here the second last inequality follows from the fact that the KL-divergence is nonnegative. Also, the last inequality follows from the initialization of the policy and Q-function in Line 1 of Algorithm \ref{ppoalgo}, which implies that $\pi^1_h(\cdot\,|\,x_h)$ is a uniform distribution on $\cA$ and hence
\$
\dkl\bigl( \pi^*_h(\cdot\,|\,x_h) \,\big\|\, \pi^1_h(\cdot\,|\,x_h) \bigr)  &= \sum_{a\in\cA} \pi^*_h(a\,|\,x_h)\cdot \log\bigl(|\cA| \cdot \pi^*_h(a\,|\,x_h)\bigr) \\
&= \log|\cA|+ \sum_{a\in\cA} \pi^*_h(a\,|\,x_h)\cdot\log\bigl(\pi^*_h(a\,|\,x_h) \bigr)  \le \log|\cA|.
\$
Here the inequality follows from the fact that the entropy of $\pi^*_h(\cdot\,|\,x_h)$ is nonnegative. Thus, setting $\alpha=\sqrt{2\log{|\cA|}/(HT)}$ in Line 6 of Algorithm \ref{ppoalgo}, by \eqref{eq:w361ok} we obtain
\#\label{1007621}
\sum_{k=1}^K\sum_{h=1}^H \EE_{\pi^*} \bigl[ \la Q^{k}_h(x_h,\cdot), \pi^*_h(\cdot\,|\,x_h) - \pi^k_h(\cdot\,|\,x_h) \ra\,\big|\, x_1 = x^k_1 \bigr]\le \sqrt{2H^3T \cdot \log|\cA|},
\#
where $T = HK$.

\vspace{4pt}
\noindent
\textbf{Term (ii).} Recall that the martingale differences $D_{k,h,1}$ and $D_{k,h,2}$ defined in \eqref{1005805} take the following forms,
\$
D_{k,h,1}&=\bigl(\mathbb{J}_{k,h} (Q^{k}_h-Q^{\pi^k,k}_h)\bigr)(x^k_h) - ( Q^{k}_h-Q^{\pi^k,k}_h)(x^k_h,a^k_h), \\
D_{k,h,2}&=\bigl(\mathbb{P}_h (V^{k}_{h+1}-V^{\pi^k,k}_{h+1})\bigr)(x^k_h,a^k_h) - (V^{k}_{h+1}-V^{\pi^k,k}_{h+1})(x^k_{h+1}).
\$
By the truncation of $Q^k_h$ to the range $[0, H-h+1]$ in \eqref{eq:w1150}, we have 
\$
Q^{k}_h,Q^{\pi^k,k}_h,V^{k}_{h+1},V^{\pi^k,k}_{h+1} \in[0,H],
\$
which implies that $|D_{k,h,1}|\le 2H$ and $|D_{k,h,2}|\le 2H$ for any $(k,h)\in[K]\times[H]$. Therefore, applying the Azuma-Hoeffding inequality to the martingale defined in \eqref{eq:w445ak}, we obtain
\$
P \bigl( |\cM_{K,H,2}|>t \bigr) \leq 2\exp\biggl( \frac{-t^2}{16H^2T} \biggr)
\$
for any $t>0$. Setting $t=\sqrt{16H^2T\cdot\log(4/\zeta)}$ with $\zeta\in (0,1]$, we obtain
\#\label{1007622}
|\cM_{K,H,2}|\le \sqrt{16H^2T\cdot\log(4/\zeta)}
\#
with probability at least $1-\zeta/2$, where $T = HK$.

\vspace{4pt}
\noindent
\textbf{Term (iii).} 
As is shown in Lemma \ref{10061229}, it holds with probability at least $1-\zeta/2$ that
\#\label{eq:apply_lemma_opt}
  - 2 \Gamma _h^k (x,a)\leq \iota_h^k(x,a) \leq 0
  \# 
   for any $(k,h)\in[K]\times[H]$ and $(x,a)\in\cS\times\cA$. Meanwhile, by the definitions of $\iota_h^k$ and $Q_h^k $ in \eqref{eq:w11260901}  and \eqref{eq:w1150}, respectively, 
we have that $|\iota_h^k(x,a) | \leq 2 H$, which together with  \eqref{eq:apply_lemma_opt} implies 
$$
- \iota_h^k (x,a) \leq 2 \min \{ H,  \Gamma_h^k (x,a) \}
$$
for any $(k,h)\in[K]\times[H]$ and $(x,a)\in\cS\times\cA$ with probability at least $1 - \zeta / 2$. 
Hence, we obtain
\#\label{eq:521lol}
\sum_{k=1}^K\sum_{h=1}^H\bigl( \EE_{\pi^*}[\iota^{k}_h(x_h,a_h)\,|\, x_1 = x^k_1] - \iota^{k}_h(x^k_h,a^k_h)\bigr) \le 2\sum_{k=1}^K\sum_{h=1}^H \min\{H, \Gamma^{k}_h(x^k_h,a^k_h)\}
\#
with probability at least $1-\zeta/2$. By the definition of $\Gamma^{k}_h$ in \eqref{1015509}, we have
\#\label{1007548}
\sum_{k=1}^K\sum_{h=1}^H \min\{ H, \Gamma^{k}_h(x^k_h,a^k_h) \}
=\beta\cdot \sum_{h=1}^H \sum_{k=1}^K \min\Bigl\{H/\beta, \sqrt{\phi^k_h(x^k_h,a^k_h)^\top (\Lambda^{k}_h)^{-1}\phi^k_h(x^k_h,a^k_h)} \Bigr\}.
\#
Recall that we set
\#\label{eq:1016ok}
\beta = C \sqrt{dH^2 \cdot \lok}
\#
with $C>1$ being an absolute constant, which implies that $H \leq \beta $. Thus, \eqref{1007548} implies
\# \label{10075480}
\sum_{k=1}^K\sum_{h=1}^H \min\{ H, \Gamma^{k}_h(x^k_h,a^k_h) \}
\leq \beta\cdot \sum_{h=1}^H \sum_{k=1}^K \min\Bigl\{1, \sqrt{\phi^k_h(x^k_h,a^k_h)^\top (\Lambda^{k}_h)^{-1}\phi^k_h(x^k_h,a^k_h)} \Bigr\}.
\#
By Lemma \ref{1007521} and the definition of $\Lambda^{k}_h$ in \eqref{1015509}, we obtain
\#\label{1007549}
\sum_{k=1}^K \min  \{1, \phi^k_h(x^k_h,a^k_h)^\top (\Lambda^{k}_h)^{-1}\phi^k_h(x^k_h,a^k_h)  \}
\le 2\log\biggl( \frac{\det(\Lambda^{K+1}_h)}{\det(\Lambda^1_h)} \biggr)
\#
for any $h\in[H]$, where $\Lambda^1_h=\lambda\cdot I$ and $\Lambda^{K+1}_h\in\cF_{K,H,2}$ by Definition \ref{def:w001}. 
Moreover, Assumption \ref{linearmdp} implies
\$
\| \phi_h^k (x,a) \|_2 \leq \sqrt{d}H
\$ 
for any $(k,h)\in[K]\times[H]$ and $(x,a)\in\cS\times\cA$, which further implies   
\$
\Lambda^{K+1}_h = \sum_{k=1}^{K} \phi^k_h(x^k_h,a^k_h)\phi^k_h(x^k_h,a^k_h)^\top + \lambda\cdot{I} \preceq (dH^2K+\lambda)\cdot I
\$
for any $h\in[H]$. As we set $\lambda=1$, it holds for any $h\in[H]$ that
\#\label{1007550}
2\log\biggl( \frac{\det(\Lambda^{K+1}_h)}{\det(\Lambda^1_h)} \biggr)
\le 
2d\cdot \log\bigl((dH^2K + \lambda)/\lambda\bigr)\le 4d \cdot \log (dHT).
\#
 Combining \eqref{10075480}-\eqref{1007550} and the Cauchy-Schwarz inequality, we obtain
\#\label{eq:w1001ok}
\sum_{k=1}^K\sum_{h=1}^H \min\{ H, \Gamma^{k}_h(x^k_h,a^k_h) \} 
&\le
\beta\cdot \sum_{h=1}^H \Bigl(K\cdot\sum_{k=1}^K \min  \{ 1, \phi^k_h(x^k_h,a^k_h)^\top (\Lambda^{k}_h)^{-1}\phi^k_h(x^k_h,a^k_h)  \} \Bigr)^{1/2}\notag\\
&  \le \beta \cdot \sum_{h = 1}^H \sqrt{K} \cdot \Biggl (2\log\biggl( \frac{\det(\Lambda^{K+1}_h)}{\det(\Lambda^1_h)} \biggr)  \Biggr )^{1/2 }\notag\\
& \leq  2 \beta \sqrt{  dH^2 K \cdot \log (dHT) }  .
\#
By \eqref{eq:521lol}, \eqref{eq:1016ok}, and \eqref{eq:w1001ok}, it holds with probability at least $1-\zeta/2$ that
\#\label{1007617}
&\sum_{k=1}^K\sum_{h=1}^H\bigl( \EE_{\pi^*}[\iota^{k}_h(x_h,a_h) \,|\, x_1 = x^k_1] - \iota^{k}_h(x^k_h,a^k_h)\bigr)\notag \\
&\quad \le 4 \beta \sqrt{  dH^2 K \cdot \log (dHT) }  \leq 8C \sqrt{  d^2 H^3  T}\cdot \lok,
\#
where $C>1$ is an absolute constant, $\zeta\in (0,1]$, and $T = HK$. 

Plugging the upper bounds of terms (i)-(iii) in \eqref{1007621}, \eqref{1007622}, and \eqref{1007617}, respectively, into \eqref{1015243} of Lemma \ref{1005415}, we obtain
\$
\text{Regret}(T) \le C'\reg\cdot \lok
\$
with probability at least $1-\zeta$, where $C'>0$ is an absolute constant. Here we use the fact that $\log |\cA| = O(d^2\cdot [\lok]^2)$ in \eqref{1007621} and \eqref{1007617}. Therefore, we conclude the proof of Theorem \ref{10061231}.
\end{proof}

%% file: aux_lemma.tex

\section{Supporting Lemmas}
In this section, we present the supporting lemmas.

\begin{lemma}\label{eventlm}
Let $\lambda = 1$ in \eqref{eq:w1150} and Line 12 of Algorithm \ref{ppoalgo}. For any $\zeta\in (0,1]$, the event $\mathcal{E}$ that, for any $(k,h)\in [K]\times[H]$,
\#\label{event}
\Bigl\| \sum_{\tau=1}^{k-1} \phi^\tau_h(x_h^\tau,a_h^\tau) \cdot \bigl( V^{\tau}_{h+1}(x^\tau_{h+1}) - (\mathbb{P}_hV^{\tau}_{h+1})(x^\tau_{h}, a^\tau_h) \bigr) \Bigr\|_{(\Lambda^{k}_h)^{-1}}
\le C'' \sqrt{dH^2\cdot\lok}
\#
happens with probability at least $1-\zeta/2$, where $C''>0$ is an absolute constant.
 \end{lemma}
\begin{proof}
By the definition of the filtration $\{\cF_{k,h,m}\}_{(k,h,m)\in[K]\times[H]\times[2]}$ in Definition \ref{def:w001} and the Markov property, we have
\#\label{200703441}
\EE[V^{\tau}_{h+1}(x^\tau_{h+1}) \,|\, \cF_{\tau,h,1}] = (\mathbb{P}_hV^{\tau}_{h+1})(x^\tau_{h}, a^\tau_h). 
\#
Conditioning on $\cF_{\tau,h,1}$, the only randomness comes from $x^\tau_{h+1}$, while $V^\tau_{h+1}$ is a deterministic function. To see this, note that $V^\tau_{h+1}$ is determined by $Q^\tau_{h+1}$ and $\pi^\tau_{h+1}$, which are further determined by the historical data in $\cF_{\tau,h,1}$. We define
\$
\eta_{\tau,h}=V^{\tau}_{h+1}(x^\tau_{h+1}) - (\mathbb{P}_hV^{\tau}_{h+1})(x^\tau_{h}, a^\tau_h).
\$
By \eqref{200703441}, conditioning on $\cF_{\tau,h,1}$,     $\eta_{\tau,h} $ is a zero-mean random variable. Moreover, as $V^\tau_{h+1}\in[0,H]$, conditioning on $\cF_{\tau,h,1}$, $\eta_{\tau,h}  $ is an $H/2$-sub-Gaussian random variable, which is defined in \eqref{eq:define_subgaussian} of Lemma \ref{1008241}. Also, $\eta_{\tau,h}$ is $\cF_{k,h,2}$-measurable,
as $x^\tau_{h+1} \in \cF_{\tau,h,2}$ for any $\tau \in [k-1]$. Hence, for any fixed $h\in[H]$, by Lemma \ref{1008241}, it holds with probability at least $1-\zeta/(2H)$ that
\#\label{200703516}
&\Bigl\| \sum_{\tau=1}^{k-1} \phi^\tau_h(x_h^\tau,a_h^\tau) \cdot \bigl( V^{\tau}_{h+1}(x^\tau_{h+1}) - (\mathbb{P}_hV^{\tau}_{h+1})(x^\tau_{h}, a^\tau_h) \bigr) \Bigr\|^2_{(\Lambda^{k}_h)^{-1}} \notag\\
&\quad \le H^2/2\cdot \Bigl( \log\bigl( \det(\Lambda^k_h)^{1/2} \det(\lambda\cdot I)^{-1/2} \bigr) + \log(2H/\zeta) \Bigr)
\#
for any $k\in[K]$. 
To upper bound $\det (\Lambda_h^k)$ in \eqref{200703516}, 
recall that $\Lambda^k_h$ is defined by
\$
\Lambda^k_h = \sum_{\tau=1}^{k-1} \phi^\tau_h(x^\tau_h, a^\tau_h) \phi^\tau_h(x^\tau_h, a^\tau_h)^\top + \lambda \cdot I.
\$
By the triangle inequality,  the spectral norm of $\Lambda^k_h$ is upper bounded as
\$
\| \Lambda^k_h \|_2 \le \lambda + \sum_{\tau=1}^{k-1} \| \phi^\tau_h(x^\tau_h, a^\tau_h)\|_2^2 \le \lambda + dH^2K.
\$
Here the last inequality follows from Assumption \ref{linearmdp},
which implies  
\$
\sup_{(x,a) \in \cS \times \cA } \Bigl \| \int_{\cS } \psi(x, a,x') \cdot V(x') \ud x' \Bigr \|_2 \leq \sqrt{d} H
\$
for any $V \colon \cS \rightarrow [0, H]$.
Hence,  $\det(\Lambda^k_h)$ in \eqref{200703516} is upper bounded as 
\#\label{200703515}
\det(\Lambda^k_h) \le \| \Lambda^k_h \|^d_2 \le ( \lambda + dH^2K)^d.
\#
Moreover, setting $\lambda = 1$, combining \eqref{200703516} and \eqref{200703515}, and applying the union bound for any $h\in[H]$, we obtain that, with probability at least $1-\zeta/2$,
\$
&\Bigl\| \sum_{\tau=1}^{k-1} \phi^\tau_h(x_h^\tau,a_h^\tau) \cdot \bigl( V^{\tau}_{h+1}(x^\tau_{h+1}) - (\mathbb{P}_hV^{\tau}_{h+1})(x^\tau_{h}, a^\tau_h) \bigr) \Bigr\|^2_{(\Lambda^{k}_h)^{-1}} \\
&\quad \le
H^2/2\cdot \Bigl(d / 2\cdot \log\bigl( (\lambda+dH^2K)/\lambda \bigr) + \log(2H/\zeta) \Bigr)
\le C''^2 dH^2 \cdot \lok
\$ 
for any $(k,h)\in[K]\times[H]$, where $C''>0$ is an absolute constant. Thus, we conclude the proof of Lemma \ref{eventlm}.
\end{proof}

\begin{lemma}[Concentration of Self-Normalized Process \citep{abbasi2011improved}]\label{1008241}
Let $\{ \tilde{\cF}_t \}^\infty_{t=0}$ be a filtration and $\{\eta_t\}^\infty_{t=1}$ be an $\RR$-valued stochastic process such that $\eta_t$ is $\tilde{\cF}_{t+1} $-measurable for any $t\geq 0$. 
Moreover, we assume that, for any $t\geq 0$, conditioning on $\tilde \cF_t$,  $\eta_t$ is a zero-mean and  $\sigma$-sub-Gaussian random variable with the variance proxy $\sigma^2 > 0$, that is,
\#\label{eq:define_subgaussian}
\EE[e^{\lambda \eta_t} \,|\, \tilde{\cF}_t ] \le e^{\lambda^2\sigma^2/2}
\#
for any $\lambda\in\RR$. Let $\{X_t\}^\infty_{t=1}$ be an $\RR^d$-valued stochastic process such that $X_t$ is $\tilde{\cF}_t$-measurable for any $t\geq 0$. Also, let  $Y \in \RR^{d\times d}$ be  a deterministic and  positive-definite matrix. For any $t\ge0$, we  define
\$
\overline{Y}_t = Y + \sum_{s=1}^t X_s X_s^\top, \quad S_t=\sum_{s=1}^t \eta_s\cdot X_s.
\$
For any $\delta>0$, it holds with probability at least $1-\delta$ that 
\$
\| S_t \|^2_{\overline{Y}^{-1}_t} \le 2\sigma^2\cdot \log\biggl( \frac{\det(\overline{Y}_t)^{1/2}\det(Y)^{-1/2}}{\delta} \biggr)
\$
for any $t\ge0$.
\end{lemma}
\begin{proof}
See Theorem 1 of \cite{abbasi2011improved} for a detailed proof.
\end{proof}

\begin{lemma}[Elliptical Potential Lemma \citep{dani2008stochastic, rusmevichientong2010linearly, chu2011contextual, abbasi2011improved, jin2019provably}]\label{1007521}
Let $\{ \phi_t \}_{t=1}^\infty$ be an $\RR^d$-valued sequence. Meanwhile, let $\Lambda_0\in\RR^{d\times d}$ be a positive-definite matrix and $\Lambda_t=\Lambda_0 + \sum_{j=1}^{t-1} \phi_j\phi_j^\top$. It holds for any $t\in \ZZ_+$ that
\$
 \sum_{j=1}^t \min\{1, \phi^\top_j \Lambda^{-1}_{j}\phi_j  \} \le
2\log\biggl( \frac{\det(\Lambda_{t+1})}{\det(\Lambda_1)} \biggr).
\$
Moreover, assuming that $\|\phi_j\|_2\le1$ for any $j\in\ZZ_+$ and $\lambda_{\min}(\Lambda_0)\ge1$, it holds for any $t\in\ZZ_+$ that
\$
\log\biggl( \frac{\det(\Lambda_{t+1})}{\det(\Lambda_1)} \biggr)  \le  
\sum_{j=1}^t \phi^\top_j \Lambda^{-1}_{j}\phi_j \le
2\log\biggl( \frac{\det(\Lambda_{t+1})}{\det(\Lambda_1)} \biggr).
\$
\end{lemma}
\begin{proof}
See Lemma 11 of \cite{abbasi2011improved} for a detailed proof.
\end{proof}


%% file: tab.tex

\section{OPPO in the Tabular Setting} \label{sec:tabular}

{\color{red} Remember to change $\phi$ to $\psi$ notation}

In this section, we customize the theory of OPPO for the tabular setting. It is worth noting that  a tabular MDP  $(\cS, \cA,  H, \cP, r)$ is a 
special case of the linear MDP introduced  in Assumption \ref{linearmdp} 
with feature mapping being the canonical basis of $\RR^{| \cS |^2 | \cA|}$, where both $| \cS|$ and $|\cA|$ are finite.  
Specifically, throughout this section, we let $d =| \cS |^2 | \cA|$ and define the feature mapping 
$\phi\colon \cS \times \cA \times \cA\rightarrow \RR^d$ 
by letting $\phi(x,a,x') = e_{(x,a,x')}$ for all  $(x,a, x' ) $. 
In other words, for any $(x,a, x' ) $ and $(\tilde x, \tilde a, \tilde x')$ in $ \cS \times \cA \times \cS$, we let  the $(\tilde x, \tilde a, \tilde x')$-th entry of $\phi(x,a, x')$ be 
\#\label{eq:tabular_feature}
[\phi(x,a, x')]_{(\tilde x, \tilde a, \tilde x')} = \ind\{x = \tilde x, a = \tilde a,  x ' = \tilde  x'  \}. 
\#
Then, the tabular MDP satisfy the definition of linear MDP by letting $\theta_h = \cP_h (\cdot \given \cdot , \cdot)$, i.e., for any $(x,a,x') \in \cS \times \cS \times \cA$, the $(x,a,x')$-th entry of $\theta_h$ is $\cP_h (x' \given x, a)$.
Then, by definition, it holds that 
$$
\cP_h(x' \given x, a) = \phi (x,a, x') ^\top \theta_h.
$$
Moreover, by direct computation, we have 
\#\label{eq:tabular_theta_norm} 
\| \theta_h \|_2 = \biggl (\sum_{(x,a,x') } | \cP_h (x' \given x, a ) | ^2 \biggr ) ^{1/2}\leq \biggl (\sum_{(x,a,x') }   \cP_h (x' \given x, a )   \biggr ) ^{1/2} \leq \sqrt{ | \cS| |\cA| }.
\#
Meanwhile, for any $V \colon \cS \rightarrow [ 0, H ]$ and any $(x,a) \in \cS \times \cA$, 
we have 
\#\label{eq:tabular_phi_norm} 
\bigg\| \sum_{x' \in \cS }\phi(x,a, x') V(x') \biggl \|_2 = \bigg (\sum_{x' \in \cS } \bigl ( V(x')\bigr ) ^2\biggr ) ^{1/2}  \leq \sqrt{|S| } H  .
\#
Notice that both \eqref{eq:tabular_theta_norm} and \eqref{eq:tabular_phi_norm} appear in Assumption \ref{linearmdp}, with the upper bounds $\sqrt{|\cS| |\cA| }$ and $\sqrt{|\cS| } H$  replaced by $\sqrt{|\cS|^2 |\cA| }$  and $\sqrt{|\cS|^2 |\cA| } H$, respectively. 
Hence,  we notice that the tabular MDP is a special case of the linear MDP in Assumption \ref{linearmdp} with $d = | \cS|^2 |\cA|$. 
 
In the following,   we  
re-state OPPO for the tabular setting. 
 Recall that in the policy evaluation step of OPPO, 
we solve the  least-squares regression problem in \eqref{eq:w1150} to obtain $w_h^k \in \RR^{d}$, where the least-squares loss $M_h^k$ is given in \eqref{eq:w230a}. 
In the tabular case, 
$\phi_h^\tau \colon \cS \times \cA  \rightarrow \RR^d$ is given by 
\#\label{eq:tabular_phi_hk}
\phi_h^\tau (x,a) = \sum_{x'} \phi(x,a, x') \cdot V_{h+1}^\tau (x')  , \qquad \forall (x,a ) \in \cS \times \cA. 
\#
Then, each entry of $\phi_h^\tau (x,a)$ is given by 
\#\label{eq:tabular_entry_phi}
[\phi_h^\tau (x,a) ]_{( \tilde x, \tilde a, \tilde x')} = \ind \{ (x, a) = (\tilde x, \tilde a)\} \cdot V_{h+1} ^\tau (\tilde x'), \qquad \forall (\tilde x, \tilde a, \tilde x') \in \cS \times \cA\times \cS. 
\#
In other words, each $\phi_h^\tau (x,a)$ is a sparse vector with $|\cS|$ nonzero entries with indices $\{ (x, a, x')\} _{x' \in \cS}$, and the $(x,a,x')$-th entry of $\phi_h^\tau (x,a)$ is $V_{h+1}^\tau (x')$. 
Combining \eqref{eq:tabular_phi_hk} and \eqref{eq:tabular_entry_phi}, 
$M_h^k$ in \eqref{eq:w230a} can be written as 
\#\label{eq:tabular_ls_loss}
M_h^k (w) & =\sum_{\tau=1}^{k-1} \bigl(V^{\tau}_{h+1}(x^{\tau}_{h+1}) - \phi^\tau_h(x^\tau_h, a^\tau_h)^\top w\bigr)^2 \notag \\
& =   \sum_{\tau =1 }^{k-1} \biggl( V_{h+1}^\tau (x_{h+1}^\tau ) - \sum_{(x,a,x')} \ind \{(x,a) = (x_h^\tau , a_h^\tau ) \} \cdot V_{h+1} ^\tau (x' )  \cdot [w]_{(x,a,x')}
\bigg) ^2 ,
\#
where $[ w]_{(x,a,x')}$ is the $(x,a,x')$-th entry of $w$.  
Let $  w_h^k$ be the solution to the regularized least-squares  regression problem in \eqref{eq:w1150} with $M_h^k$ given in \eqref{eq:tabular_ls_loss}. 
Here each $w_h^k$ is an estimator of the transition kernel $P_h (\cdot \given \cdot ,\cdot )$. 
Moreover, we define $\Lambda_h^k $ and the bonus function $\Gamma_h^k\colon \cS \times \cA \rightarrow \RR$ the same as in \eqref{1015509}.  
To complete the policy evaluation step, we update the value functions 
by letting 
\#\label{eq:tabular_q_update}
Q^{k}_h(x,a) & =  \min\{r^k_h(x ,a ) + \phi^k_h(x , a )^\top w^k_h + \Gamma^k_h(x, a ), H-h+1\}^+\\
&  = \min \Big \{ r^k_h(x ,a )  + \sum_{x' \in \cS} V_{h+1} ^k(x')  \cdot [ w_h^k ]_{(x,a,x')} + \Gamma^k_h(x, a ), H - h + 1 \Big \}^{+} , \quad \forall (x,a)\in \cS \times \cA,  \notag 
\#
and letting $V_{h}^k (x) = \la Q_h^k (x, \cdot ), \pi_h^k (\cdot \given x) \ra_{\cA}$. Then, based on $\{Q_{h}^k \} _{h\in [H]}$, in the policy improvement step of OPPO, we update the policy by 
\$
\pi_h^{k+1} (\cdot \given x) \propto \pi_h^k (\cdot \given x) \cdot \exp \bigl ( \alpha \cdot Q_h^{k} (x, \cdot )\bigr ) .
\$ 

In the following corollary, we apply   Theorem \ref{10061231} to the tabular setting and establish the regret of OPPO.

\begin{corollary} [Regret in the Tabular Setting] \label{cor:tabular} 
    Under the tabular setting, let the feature mapping $\phi$  be the canonical basis given  \eqref{eq:tabular_feature}. 
     Let $\alpha=\sqrt{2\log{|\cA|}/(HT)}$ in \eqref{1014522} and Line \ref{policyupdateline} of Algorithm \ref{ppoalgo}, $\lambda=1$ in \eqref{eq:w1150} and Line 12 of Algorithm \ref{ppoalgo}, and $$\beta=C\sqrt{| \cS|^2 | \cA| H^2\log(| \cS|  | \cA|HT/\zeta)}$$ in \eqref{1015509} and Line 15 of Algorithm \ref{ppoalgo},
      where $C>0$ is an absolute constant and $\zeta\in (0,1]$. Then, with probability at least $1-\zeta$, the regret of OPPO satisfies 
    \$
    \text{Regret}(T) \le C'\cdot | \cS|^2 | \cA| \cdot \sqrt{H^3 T}\cdot \log( | \cS|  | \cA| HT/\zeta), 
    \$
  where $C' > 0$ is an absolute constant. 
\end{corollary}

\begin{proof}
    Since the tabular MDP is a special case of the linear MDP introduced in Assumption \ref{linearmdp} with $d = |\cS|^2 |\cA|$. 
    Corollary \ref{cor:tabular} directly follows from Theorem \ref{10061231}  with $d$ replaced by $|\cS|^2 |\cA|$. 
\end{proof}

\subsection{A Variant of OPPO for the Tabular Setting}
As we have shown in \S\ref{sketch}, 
the  regret decomposition of  OPPO given in Lemma \ref{1005415} holds beyond the 
linear MDP model. 
Moreover, combining  Lemma \ref{10061229}, we observe that, as long as the principle of "optimism in the face of uncertainty" is ensured such that (1) the model prediction error $\iota_h^k$ defined in \eqref{eq:w11260901} is dominated by the bonus function $\Gamma_h^k$  and  (2) the sum of bonus terms 
$\sum_{k=1}^K \sum_{k=1}^H \Gamma_h^k (x_h^k, a_h^k ) $  is bounded, 
the  analysis framework for OPPO would yield an upper bound on the regret. 
Such   flexibility enables OPPO to readily incorporate other optimistic policy evaluation methods besides the one introduced in \eqref{eq:w230a}--\eqref{1015509}. 
To demonstrate such a property, in the sequel,
we introduce a variant of OPPO for the tabular setting which is based on another policy evaluation approach.

Specifically, for any $(h, k) \in [H]\times [K]$ and any $(x, a, x') \in \cS \times \cA \times \cS$, we define 
 $n^k_h(x,a,x')$ and $n^k_h(x,a)$ by 
\#\label{eq:tab_counts}
n^k_h(x,a,x') = \sum_{\tau=1}^{k-1} \ind\{(x^\tau_h,a^\tau_h,x^\tau_h)=(x,a,x')\}, \quad 
n^k_h(x,a) = \sum_{\tau=1}^{k-1} \ind\{(x^\tau_h,a^\tau_h)=(x,a)\} .
\# 
In other words, $n^k_h(x, a)$ counts the number of times we observe the state-action pair $(x,a)$ at the $h$-th step before the $k$-th episode, while $n^k_h(x,a,x')$ counts the number of times we observe the state transition $(x,a,x')$ at the $h$-th step before the $k$-th episode. 
Then we define $\hat \cP_{k,h}$ as an estimator of $\cP_h$ by 
\# \label{eq:tab_Phat}
\hat \cP_{k,h } (x' \given x, a) = \frac{ n_h^k(x, a, x')}{n_h^k (x,a) + \lambda}, \quad \forall (x, a, x' ) \in \cS \times \cA \times \cS, 
\#
where $\lambda >0$ is the regularization parameter. 
Meanwhile, we define the bonus function $\Gamma_h^k \colon \cS \times \cA \rightarrow \RR$ by letting 
\# \label{eq:tab_bonus}
\Gamma_h^k (x,a ) = \beta \cdot \bigl ( n_h^k (x,a) + \lambda \bigr )^{-1/2},
\#
which is the  count-based bonus  commonly used in the literature \citep{azar2017minimax, jin2018q}. 
Here  $\beta > 0$ in \eqref{eq:tab_bonus} is a parameter to be determined later. 
Then, when solving the policy evaluation problem for $\pi^k$, 
based on the estimated transition kernels $\{\hat \cP_{k,h} \}_{h \in [H]}$, reward functions $\{r_h^k \}_{h\in [H]}$ received by the agent, and the bonus functions $\{ \Gamma_h^k \}_{h \in [H]}$, we estimate $\{ Q_h^{\pi^k, k}, V_h^{\pi^k, k}\}_{h \in [H]}$ via an optimistic version of dynamic programming. 
Specifically, 
we let $V_{H+1}^k$ be a zero function. 
Then for any $h \in [H]$, 
we define 
\#\label{eq:tab_define_q}
\begin{split}
Q_h^k (x  , a ) & = \min \Bigl \{ r_h^k (x,a) + \sum_{x'\in \cS} \hat \cP_{k,h}(x' \given x, a) \cdot V_{h+1}^k(x') + \Gamma_h^k (x, a), H- h +1\Bigr \}^{+},  \\
V_h^k(x)& = \la Q_{h}^k(x, \cdot ), \pi_h^k(\cdot \given x) \ra_{\cA},   \qquad \qquad \forall(x,a)\in \cS\times \cA.
\end{split}
\#
Finally, based on $Q_h^k$, we obtain $\pi_h^{k+1}$  via 
\$
\pi^{k+1}_h (\cdot \given x) \propto \pi_h^{k} (\cdot \given x) \cdot \exp \big( \alpha \cdot Q_h^k(x, \cdot )\bigr ), \quad \forall x \in \cS. 
\$
We present the details of such a variant of OPPO in Algorithm \ref{ppoalgo2}.  

\begin{algorithm}[t]
    \caption{A Variant of OPPO for the Tabular Case}
    \begin{algorithmic}[1]
    \STATE Initialize $\{\pi^0_h(\cdot \,|\, \cdot)\}_{h=1}^H$ as uniform distributions on $\cA$ and $\{Q^0_h(\cdot, \cdot)\}_{h=1}^H$ as zero functions.\label{line:winit}
    \STATE \textbf{For} episode $k=1,2,\ldots, K$ \textbf{do}
    \STATE \hspace{0.15in} Receive the initial state $x_1^k$.
    \STATE \hspace{0.15in} \textbf{For} step {$h=1, 2, \ldots, H$} \textbf{do} \label{line:pis-start}
    \STATE \hspace{0.30in} Update the policy by
    \STATE \hspace{0.40in} $\pi^k_h(\cdot\,|\,\cdot) \propto \pi^{k-1}_h(\cdot\,|\,\cdot) \cdot \exp\{\alpha\cdot Q^{k-1}_h(\cdot,\cdot)\}$. \label{policyupdateline}
    \STATE \hspace{0.30in} Take the action following $a^k_{h}\sim\pi^k_h(\cdot\,|\,x_h^k)$.
    \STATE \hspace{0.30in} Observe the reward function $r^k_{h}(\cdot,\cdot)$.
    \STATE \hspace{0.30in}  Receive the next state $x^k_{h+1}$. \label{line:pis-end}
    \STATE \hspace{0.15in} Initialize $V^k_{H+1}(\cdot)$ as a zero function.
    \STATE \hspace{0.15in} \textbf{For} step {$h=H, H-1,\ldots, 1$} \textbf{do}\label{line:pes-start}
    \STATE \hspace{0.30in} Define $\hat P_{k,h}$ and $\Gamma_h^k$ as in \eqref{eq:tab_Phat} and \eqref{eq:tab_bonus}, respectively. 
    \STATE \hspace{0.30in}  
    $Q_h^k \leftarrow  \min \{ r^k_h(\cdot,\cdot) + \sum_{x'\in \cS} \hat \cP_{k,h}(x' \given \cdot,\cdot) \cdot  V_{h+1} ^k (x')  + \Gamma^k_h(\cdot,\cdot) , H - h  + 1\}^{+}$
    \STATE \hspace{0.30in} $V^k_h(\cdot)\leftarrow \la Q^k_h(\cdot,\cdot), \pi^k_h(\cdot\,|\,\cdot) \ra_\cA$.
    \end{algorithmic}\label{ppoalgo2}
    \end{algorithm} 
    
In the following theorem, we establish an upper bound on the regret of Algorithm \ref{ppoalgo2}.

\begin{theorem}\label{thm:tab} 
  In the modified modified version of OPPO given  in Algorithm \ref{ppoalgo2}, we set     
 $\alpha=\sqrt{2\log{|\cA|}/(HT)}$,  $\lambda=1$,  and $\beta=C\sqrt{dH^2\cdot\log(dHT/\zeta)}$, where $C>1$ is an absolute constant and $\zeta\in (0,1]$.  
   Then, for any $\zeta \in (0,1)$, the regret of Algorithm \ref{ppoalgo2}
    \$
    \text{Regret}(T) \le C'\sqrt{ | \cS|^2 |\cA| H^3 T}\cdot \log(|\cS| |\cA| HT/\zeta)
    \$
    with probability at least $1-\zeta$, where $C' > 0$ is an absolute constant. 
    \end{theorem}

    \begin{proof}
        The proof of Theorem  \ref{thm:tab} is similar to that of Theorem \ref{10061231}.
        Recall that we define the model prediction error $\iota_h^k$ in \eqref{eq:w11260901}. 
        By  the regret decomposition given in Lemma \ref{1005415}, we have 
        \# \label{eq:tabular1}
        \text{Regret}(T) & = \sum_{k=1}^K \bigl(V^{\pi^*,k}_1(x^k_1) - V^{\pi^k,k}_1(x^k_1)\bigr) \notag  \\
        &=  \underbrace{\sum_{k=1}^K\sum_{h=1}^H \EE_{\pi^*} \bigl[ \la Q^{k}_h(x_h,\cdot), \pi^*_h(\cdot\,|\,x_h) - \pi^k_h(\cdot\,|\,x_h) \ra  \given x_1 = x_1^k \bigr]}_{\dr (i)}   + \underbrace{ \cM_{K, H, 2}}_{\dr (ii)} \notag\\
        &\qquad+\underbrace{ \sum_{k=1}^K\sum_{h=1}^H\bigl( \EE_{\pi^*}[\iota^{k}_h(x_h,a_h) \given x_1 = x_1^k ] - \iota^{k}_h(x^k_h,a^k_h)\bigr)}_{\dr (iii)},
        \#
        where 
        $\{\cM_{k,h,m}\}_{(k,h,m)\in[K]\times[H]\times[2]}$ form a martingale adapted to the filtration $\{\cF_{k,h,m}\}_{(k,h,m)\in[K]\times[H]\times[2]}$, both with respect to the timestep index $t(k,h,m)$ defined in \eqref{eq:w1155w}.
        See  Definition \ref{def:w001} for the details of the filtration. 
        Then, as shown in  \eqref{eq:w361ok} and \eqref{1007621}, by setting $\alpha = \sqrt{ 2 \log |\cA| / (H^2 K)}$, it holds that 
        \#\label{eq:tabular2}
        \mathrm{(i)} = \sum_{k=1}^K\sum_{h=1}^H \EE_{\pi^* } \bigl[ \la Q^{k}_h(x_h,\cdot), \pi^*_h(\cdot\,|\,x_h) - \pi^k_h(\cdot\,|\,x_h) \ra \given x_1 = x_1^k  \bigr] \leq \sqrt{ 2 H^3 T \cdot \log | \cA| }.
        \#
        Meanwhile, as we show in \eqref{1007622}, applying the Azuma-Hoeffding inequality to term (ii) in \eqref{eq:tabular1}, 
        we obtain that 
        \#\label{eq:tabular3}
        | \cM_{K, H , 2} | \leq \sqrt{ 16 H^2 T \cdot \log (4 / \zeta )}
        \# 
         holds with probability at least $1 - \zeta / 2$. 
        
         Thus, it remains to handle the last term in \eqref{eq:tabular1}. 
         To this end, we modify Lemma \ref{10061229} for the tabular setting and obtain the following lemma. 
         
        \begin{lemma}[Upper Confidence Bound] \label{lemma:tabular_optimism1}
            Recall that we define the bonus function $\Gamma_h^k$ in \eqref{eq:tab_bonus}.
            We set $\lambda=1$ and $\beta=C\sqrt{H^2\log(dHT/\zeta)}$ in \eqref{eq:tab_bonus},
            where  where $C>1$ is an absolute constant and $\zeta\in (0,1]$. 
            Consider the model prediction error $\iota_h^k$ defined in \eqref{eq:w11260901} where 
              the values functions $\{ Q_h^k, V_h^k , (k,h) \in [K] \times [K]\}$ are constructed by Algorithm \ref{ppoalgo2}.
              Then, with 
              probability at least $1-\zeta/2$,  we have 
              \$
             -2\Gamma^{k}_h(x,a) \le \iota^{k}_h(x,a) \le 0
             \$
            for any $(k,h)\in[K]\times[H]$ and $(x,a)\in\cS\times\cA$.
            \end{lemma}
            
            \begin{proof}
            Note that $\hat \cP_{k,h}$ in \eqref{eq:tab_Phat} is an estimator of $\cP_h$.  
            To simplify the notation, 
            we define $\cV = \{V \colon \cS \rightarrow [0, H] \}    $ as the set of bounded functions on $\cS$.
For any fixed function  $V \in \cV$, we consider the difference between $\sum_{x'\in \cS} \hat \cP_{k,h} (x'\given \cdot, \cdot ) \cdot V(x')$ and $\sum_{x'\in \cS} \cP_h(x' \given \cdot , \cdot ) \cdot V(x')$.
To this end, for any $(k, h ) \in [K]\times [H]$ and any $(x,a) \in \cS\times \cA$,  
we have 
\#\label{eq:tab_concen1}
& \bigl( n^k_h(x,a)+\lambda\bigr)^{1/2} \cdot \Bigl|
\sum_{x'\in\cS}  \big ( \hat \cP_{k,h}(x'\given x, a) \cdot V(x') -  \cP_{h}(x'\given x, a) \cdot V(x') \bigr ) \Bigr |  \notag \\
&\quad   =  \bigl (n^k_h(x,a) +\lambda \bigr)^{-1/2} \cdot \Bigl| 
\sum_{x'\in\cS} n^k_h(x,a,x') \cdot V(x') -  \bigl(n^k_h(x,a)+\lambda\bigr) \cdot (\mathbb{P}_hV)(x,a)
\Bigr|  \notag \\
&\quad\le \Bigl ( \sum_{\tau=1}^{k-1}  \ind\{(x^\tau_h,a^\tau_h)=(x,a)\} + \lambda \Bigr )^{-1/2} \cdot 
\Bigl|  
\sum_{\tau=1}^{k-1} \ind\{(x^\tau_h,a^\tau_h)=(x,a)\} \cdot \bigl( V(x^\tau_{h+1}) - (\mathbb{P}_hV)(x,a) \bigr)
\Bigr| \notag \\
&\quad\qquad +  \bigl (n^k_h(x,a) +\lambda \bigr)^{-1/2} \cdot \Bigl| 
\lambda \cdot (\mathbb{P}_hV)(x,a)
\Bigr|  ,
\#
where we utilize \eqref{eq:tab_counts} and the inequality uses the triangle inequality. 
Recall that we define the filtrations $\cF_{k,h,1}$ and $\cF_{k,h,2}$ in Definition \ref{def:w001}. 
Conditioning on $\cF_{\tau, h, 1}$, it holds that 
\$
\eta_{\tau, h} = V(x_{h+1}^\tau) - (\PP_h V) (x_h^\tau, a_h^\tau) 
\$
is a zero-mean and $H/2$-sub-Gaussian random variable. 
Applying Lemma \ref{1008241} with $Y = \lambda$ and $X_{\tau } = \ind \{ (x^\tau_h,a^\tau_h)=(x,a) \}$ for all $h\in[H]$ and combining \eqref{eq:tab_concen1}, we obtain that, with probability at least $1-\delta$, for all $(k,h)\in[K]\times[H]$,
\#\label{200705246}
& \Bigl ( \sum_{\tau=1}^{k-1}  \ind\{(x^\tau_h,a^\tau_h)=(x,a)\} + \lambda \Bigr )^{-1/2} \cdot 
\Bigl|  
\sum_{\tau=1}^{k-1} \ind\{(x^\tau_h,a^\tau_h)=(x,a)\} \cdot \bigl( V(x^\tau_{h+1}) - (\mathbb{P}_hV)(x,a) \bigr)
\Bigr| \notag \\
&\quad \le \sqrt{H^2/2\cdot \log \Bigl (   \bigl ( n_h^k(x,a) +\lambda\bigr )^{1/2}\lambda^{-1/2} \big /  (\delta/H) \Bigr ) } 
\le \sqrt{H^2/2\cdot\log(T/\delta)}.
\#
Also, since $V(\cdot)\in[0,H]$, we have
\#\label{200705247}
\bigl (n^k_h(x,a) +\lambda \bigr)^{-1/2} \cdot | 
\lambda \cdot (\mathbb{P}_hV)(x,a)|
 \le \sqrt{\lambda}H.
\#
Combining \eqref{eq:tab_concen1}, \eqref{200705246}, and \eqref{200705247}, and setting $\lambda=1$, we obtain,  with probability at least $1 - \delta $, for all $k \geq 1$
\# \label{eq:tab_concen2}
&\bigl(n^k_h(x,a) + \lambda \bigr) \cdot \Bigl| 
\sum_{x'\in\cS}  \big ( \hat \cP_{k,h}(x'\given x, a) \cdot V(x') -  \cP_{h}(x'\given x, a) \cdot V(x') \bigr ) \Bigr | ^2 \leq H^2 \cdot \bigl(\log(T/\delta) + 2\bigr).
\#

Furthermore, 
we define a distance $d(V,V')=\max_{x\in\cS}|V(x)-V'(x)|$ on $\cV$. 
For any $\varepsilon > 0$,  by Lemma \ref{1008415},
there exists an $\varepsilon$-covering $\cV_{\varepsilon}$ of $\cV$ with respect to distance $d(\cdot, \cdot)$. Moreover, the cardinality of $\cV_{\epsilon}$ satisfies 
$|\cV_\varepsilon|\le (1+2\sqrt{|\cS|} H/\varepsilon)^{|\cS|}$. 
Thus, for any $V \in \cV$, 
there exists $V'\in\cV_\varepsilon$ such that 
\$
\max_{x\in\cS}|V(x)-V'(x)|\le \varepsilon.\$
By triangle inequality we obtain
\#\label{eq:tab_concen3}
&\bigl( n^k_h(x,a)+\lambda\bigr)^{1/2} \cdot \Bigl| \sum_{x'\in\cS}  \big ( \hat \cP_{k,h}(x'\given x, a) \cdot V(x') -  \cP_{h}(x'\given x, a) \cdot V(x') \bigr ) \Bigr |  \notag\\
&\quad\le 
\bigl( n^k_h(x,a)+\lambda\bigr)^{1/2} \cdot \Bigl| \sum_{x'\in\cS}  \big ( \hat \cP_{k,h}(x'\given x, a) \cdot V'(x') -  \cP_{h}(x'\given x, a) \cdot V'(x') \bigr ) \Bigr |  \notag \\
&\quad\qquad + \bigl( n^k_h(x,a)+\lambda\bigr)^{1/2} \cdot \Bigl| \sum_{x'\in\cS}  \Big ( \hat \cP_{k,h}(x'\given x, a) \cdot \bigl(V(x')-V'(x')\bigr) -  \cP_{h}(x'\given x, a) \cdot \bigl(V(x')-V'(x')\bigr) \Bigr ) \Bigr |  \notag \\
&\quad\le
\bigl( n^k_h(x,a)+\lambda\bigr)^{1/2} \cdot \Bigl| \sum_{x'\in\cS}  \big ( \hat \cP_{k,h}(x'\given x, a) \cdot V'(x') -  \cP_{h}(x'\given x, a) \cdot V'(x') \bigr ) \Bigr | 
+ 2\bigl( n^k_h(x,a)+\lambda\bigr)^{1/2}\cdot \varepsilon.
\#
Now we set $\delta $ in \eqref{eq:tab_concen2} as 
\$
\delta = \zeta /2 \cdot \bigl (| \cV_{\varepsilon} |  \cdot | \cS | |\cA|  ) ^{-1}
\$
and take a union bound over $V \in \cV_{\varepsilon}$ and  $(x,a)
 \in \cS \times \cA$.
Then by \eqref{eq:tab_concen2} we obtain that, with probability at least $1- \zeta / 2$, for all $(x,a) \in \cS\times \cA$, 
$(k,h) \in [K] \times [H]$, 
\#\label{eq:tab_concen4}
& \sup_{V \in \cV_{\varepsilon} } \bigl(n^k_h(x,a)+\lambda\bigr)^{1/2} \cdot \Bigl| 
\sum_{x'\in\cS}  \big ( \hat \cP_{k,h}(x'\given x, a) \cdot V(x') -  \cP_{h}(x'\given x, a) \cdot V(x') \bigr ) \Bigr | \notag \\
& \quad \leq  \sqrt{ H^2 \cdot \bigl(\log(T/\delta) + 2\bigr)} \le  \sqrt{ 2H^2 \cdot \big(  \log  ( |\cV_{\varepsilon})| ) + \log ( 2  |\cS| |\cA| T / \zeta ) + 2 \bigr )  } .
\#

 Now we combine \eqref{eq:tab_concen3} and \eqref{eq:tab_concen4} and  set $\varepsilon = H / K$ to obtain that, with probability at least $1- \zeta / 2$, 
  \#\label{eq:tab_concen5}
& \sup_{V \in \cV}  \bigl(n^k_h(x,a)+\lambda\bigr)^{1/2} \cdot \Bigl| 
\sum_{x'\in\cS}  \big ( \hat \cP_{k,h}(x'\given x, a) \cdot V(x') -  \cP_{h}(x'\given x, a) \cdot V(x') \bigr ) \Bigr | \notag \\
&\quad  \leq  \sqrt{ 2H^2 \cdot \big(  \log  ( |\cV_{\varepsilon})| ) + \log ( 2  |\cS| |\cA| T / \zeta ) + 2 \bigr )  }  + 2\bigl( n_h^k (x,a) +\lambda\bigr)^{1/2} \cdot H /K \notag \\
& \quad \leq C' H \cdot \sqrt{ | \cS| \cdot \log ( | \cS| | \cA| T /\zeta  ) }
  \#
  holds uniformly for all $(  k, h)$ and $(x,a)$. 
  Here $C' \geq 1$ is an absolute constant. 
 
  Recall that we define the value functions in \eqref{eq:tab_define_q}. 
  Let  $$\beta = C H  \cdot \sqrt{ | \cS| \cdot \log ( | \cS| | \cA| T /\zeta  ) } $$ 
  in \eqref{eq:tab_bonus}, where $C \geq C'$ is an absolute constant.
  By \eqref{eq:tab_concen5}, with probability at least $1 - \zeta / 2$, 
for any $(x,a) \in \cS \times \cA$ and any $(k,h) \in [K] \times [H]$, 
we have 
\#\label{eq:tab_optim5}
 &\Bigl| 
\sum_{x'\in\cS}   \hat \cP_{k,h}(x'\given x, a) \cdot V_{h+1}^k (x') -  (\PP_h V_{h+1}^k ) (x,a) \Bigr | \notag \\
&\quad \leq C' H   \cdot \sqrt{ | \cS| \cdot \log ( | \cS| | \cA| T /\zeta  ) } \cdot \bigl(n^k_h (x,a) +\lambda \bigr)^{-1/2} \leq \Gamma_h^k (x,a).
\#
  
For the model prediction error $\iota^{k}_h$ defined in \eqref{eq:w11260901}, by \eqref{eq:tab_optim5},  we have
        \#\label{eq:tab_optim6}
        -\iota^{k}_h(x,a)&=Q^{k}_h(x,a) - (r^k_h + \mathbb{P}_h V^{k}_{h+1} )(x,a)\notag\\
&\le
r^k_h(x,a)+ \sum_{x'\in \cS } \hat \cP_{k,h} (x'\given x, a) \cdot V_{h+1}^k (x') + \Gamma^{k}_h(x,a) - (r^k_h + \mathbb{P}_h V^{k}_{h+1} )(x,a) \le  2\Gamma^{k}_h(x,a)   
        \#
        for any $(x,a)\in\cS\times\cA$.  
         Meanwhile,  \eqref{eq:tab_define_q} implies that  $V^k_{h+1}(\cdot)\in[0,H-h]$ and  
        \$
        r^k_h(x,a) + ( \mathbb{P}_h V^{k}_{h+1}) (x,a)\leq 1 + (H-h)= H-h+1.
        \$
        Then \eqref{eq:tab_optim6} further implies that 
        \#\label{eq:tab_optim7}
        \iota^{k}_h(x,a)&=(r^k_h + \mathbb{P}_h V^{k}_{h+1} )(x,a)-Q^{k}_h(x,a)\notag\\
        &\le
        (r^k_h+\mathbb{P}_h V^{k}_{h+1}) (x,a) - 
         \min\Bigl \{r^k_h(x,a) + \sum_{x' \in \cS} \hat \cP_{k,h} (x'\given x, a) \cdot V_{h+1}^k(x') + \Gamma^{k}_h(x,a), H - h+1 \Bigr \} \notag\\
        &\leq \max\Bigl \{ ( \mathbb{P}_h V^{k}_{h+1}) (x,a)-\sum_{x' \in \cS} \hat \cP_{k,h} (x'\given x, a) \cdot  V_{h+1}^k(x')  - \Gamma^{k}_h(x,a), 0  \Bigr \}\le0
        \#
        for any $(x,a)\in\cS\times\cA$ under the inequality \eqref{eq:w1048k}. 
        Notice that \eqref{eq:tab_optim6} and \eqref{eq:tab_optim7}
        hold uniformly for all $(k,h) \in [K]\times [H]$ and $(x,a) \in \cS \times \cA$. Therefore, we conclude the proof of Lemma \ref{lemma:tabular_optimism1}. 
        \end{proof} 
        
      By Lemma \ref{lemma:tabular_optimism1}, with probability at least $1- \zeta /2 $,  term (iii) in \eqref{eq:tabular1}  is bounded by 
      \#\label{eq:tabular4}
      \mathrm{(iii)} \leq - \sum_{k=1}^K \sum_{h=1}^H \iota_h^k (x_h^k , a_h^k) \leq 2 \sum_{k=1}^K \sum_{h=1}^H \Gamma_h^k  (x_h^k , a_h^k) = 2 \beta \cdot\sum_{k=1}^K \sum_{h=1}^H \bigl ( n_h^k  (x_h^k , a_h^k)  + \lambda \bigr)^{-1/2},
      \#
      where the last equality follows from the definition of $\Gamma_h^k$ in \eqref{eq:tab_bonus}.

       In the following, we utilize the elliptical potential lemma (Lemma \ref{1007521}) to obtain an upper bound on the last term in \eqref{eq:tabular4}. 
       To this end, for any $(x,a) \in \cS \times \cA$, let $\overline \phi\colon \cS\times \cA \rightarrow \RR^{| \cS| |\cA|} $ be the mapping that sends $(x,a)$ to the canonical basis $e_{x,a}$ of $\RR^{|\cS||\cA|}$. 
       For any $(k,h) \in [K]\times [H]$, we  define 
       \#\label{eq:new_Lambda_mat}
       \overline \Lambda_h^k = \lambda \cdot I + \sum_{\tau = 1}^{k-1} \overline{\phi}(x_h^\tau, a_h^\tau) \overline{\phi}(x_h^\tau, a_h^\tau)^\top \in \RR^{|\cS||\cA| \times | \cS| |\cA| }.
       \#
 It can be verified that $\overline \Lambda_h^k$ is a diagonal matrix whose $(x,a)$-th diagonal entry is $n_h^k(x,a) + \lambda$. 
 Using $\overline \Lambda_h^k$ and $\overline \phi$, the bonus function $\Gamma_h^k $ can be equivalently written as 
 \#\label{eq:tabular5}
\Gamma_h^k (x,a) = \beta \cdot \sqrt{ \overline \phi(x,a)^\top (\overline \Lambda_h^k )^{-1} \overline \phi(x,a)}    .
 \#
Then combining \eqref{eq:tabular4} and  \eqref{eq:tabular5}, we have 
\#\label{eq:tabular6}
\mathrm{(iii)} \leq 2\beta \cdot \sum_{k=1}^K \sum_{h=1}^H \sqrt{ \overline \phi(x_h^k,a_h^k)^\top (\overline \Lambda_h^k )^{-1} \overline \phi(x_h^k,a_h^k)} \leq 2 \beta \cdot \sum_{h=1}^H \biggl( K \cdot \sum_{k=1}^K \overline \phi(x_h^k,a_h^k)^\top (\overline \Lambda_h^k )^{-1} \overline \phi(x_h^k,a_h^k) \bigg)^{1/2},
\#  
where the last inequality follows from the  Cauchy-Schwarz inequality. 
Moreover, by the definition of $\overline \Lambda_h^k$ in \eqref{eq:new_Lambda_mat}, we have $\overline \Lambda^{1}_h = \lambda\cdot I$ and  
\$
\overline \Lambda^{K+1}_h = \sum_{k=1}^{K} \overline \phi(x^k_h,a^k_h) \overline\phi(x^k_h,a^k_h)^\top + \lambda\cdot{I} \preceq (K+\lambda)\cdot I.
\$
Thus, combining \eqref{eq:tabular6} and Lemma  \ref{1007521}, 
 we have
\#\label{eq:tabular7}
\mathrm{(iii)} &  \leq  2 \beta \cdot \sum_{h=1}^H \biggl( K \cdot \sum_{k=1}^K \overline \phi(x_h^k,a_h^k)^\top (\overline \Lambda_h^k )^{-1} \overline \phi(x_h^k,a_h^k) \bigg)^{1/2} \notag \\
&\leq 2\beta \sqrt{K} \cdot \sum_{h=1}^H    \log ^{1/2} \biggl( \frac{\det( \overline \Lambda^{K+1}_h)}{\det(\overline  \Lambda^{1}_h)} \biggr)  \le 2\beta \sqrt{K} H \cdot  \log^{1/2} \biggl( \frac{\det\bigl((K+\lambda)\cdot{I}\bigr)}{\det(\lambda\cdot{I})} \biggr) \notag \\
& \leq 2 \beta \sqrt{|\cS| |\cA| HT} \cdot \sqrt{  
  \log\bigl((K+\lambda)/\lambda\bigr) } .
\#
Notice that we set $\lambda =1$ and let $\beta= C H \cdot \sqrt{|\cS| \cdot \log ( | \cS| |\cA| T /\zeta)}$ 
for some constant $C > 1$. 
Combining \eqref{eq:tabular1}, \eqref{eq:tabular2}, \eqref{eq:tabular3}, and \eqref{eq:tabular7}, we have 
\$
\mathrm{Regret}(T) & \leq \sqrt{2 H^3 T \cdot \log |\cA|  } + \sqrt{16 H^2 T \cdot \log (4 / \zeta)} + 2 C   \sqrt{|\cS|^2 |\cA| H^3 T} \cdot \log ( | \cS| |\cA| T /\zeta) \\
& \leq C'' \sqrt{|\cS|^2 |\cA| H^3 T} \cdot \log ( | \cS| |\cA| T /\zeta) 
\$
with probability at least $1 - \zeta$, where $C''$ is an absolute  constant. 
Therefore, we conclude the proof of Theorem \ref{thm:tab}.
\end{proof}

%% file: OPPO_arxiv.bbl
\begin{thebibliography}{77}
\expandafter\ifx\csname natexlab\endcsname\relax\def\natexlab#1{#1}\fi
\expandafter\ifx\csname url\endcsname\relax
  \def\url#1{\texttt{#1}}\fi
\expandafter\ifx\csname urlprefix\endcsname\relax\def\urlprefix{}\fi

\bibitem[{Abbasi-Yadkori et~al.(2019{\natexlab{a}})Abbasi-Yadkori, Bartlett,
  Bhatia, Lazic, Szepesv{\'a}ri and Weisz}]{abbasi2019politex}
\text{Abbasi-Yadkori, Y.}, \text{Bartlett, P.}, \text{Bhatia, K.}, \text{Lazic,
  N.}, \text{Szepesv{\'a}ri, C.} and \text{Weisz, G.} (2019{\natexlab{a}}).
\newblock {POLITEX}: Regret bounds for policy iteration using expert
  prediction.
\newblock In \textit{International Conference on Machine Learning}.

\bibitem[{Abbasi-Yadkori et~al.(2019{\natexlab{b}})Abbasi-Yadkori, Lazic,
  Szepesvari and Weisz}]{abbasi2019exploration}
\text{Abbasi-Yadkori, Y.}, \text{Lazic, N.}, \text{Szepesvari, C.} and
  \text{Weisz, G.} (2019{\natexlab{b}}).
\newblock Exploration-enhanced {POLITEX}.
\newblock \textit{arXiv preprint arXiv:1908.10479}.

\bibitem[{Abbasi-Yadkori et~al.(2011)Abbasi-Yadkori, P{\'a}l and
  Szepesv{\'a}ri}]{abbasi2011improved}
\text{Abbasi-Yadkori, Y.}, \text{P{\'a}l, D.} and \text{Szepesv{\'a}ri, C.}
  (2011).
\newblock Improved algorithms for linear stochastic bandits.
\newblock In \textit{Advances in Neural Information Processing Systems}.

\bibitem[{Agarwal et~al.(2019)Agarwal, Kakade, Lee and
  Mahajan}]{agarwal2019optimality}
\text{Agarwal, A.}, \text{Kakade, S.~M.}, \text{Lee, J.~D.} and \text{Mahajan,
  G.} (2019).
\newblock Optimality and approximation with policy gradient methods in {M}arkov
  decision processes.
\newblock \textit{arXiv preprint arXiv:1908.00261}.

\bibitem[{Antos et~al.(2008)Antos, Szepesv{\'a}ri and Munos}]{antos2008fitted}
\text{Antos, A.}, \text{Szepesv{\'a}ri, C.} and \text{Munos, R.} (2008).
\newblock Fitted {Q}-iteration in continuous action-space mdps.
\newblock In \textit{Advances in Neural Information Processing Systems}.

\bibitem[{Auer et~al.(2002)Auer, Cesa-Bianchi and Fischer}]{auer2002finite}
\text{Auer, P.}, \text{Cesa-Bianchi, N.} and \text{Fischer, P.} (2002).
\newblock Finite-time analysis of the multiarmed bandit problem.
\newblock \textit{Machine Learning}, \textbf{47} 235--256.

\bibitem[{Ayoub et~al.(2020)Ayoub, Jia, Szepesvari, Wang and
  Yang}]{ayoub2020model}
\text{Ayoub, A.}, \text{Jia, Z.}, \text{Szepesvari, C.}, \text{Wang, M.} and
  \text{Yang, L.} (2020).
\newblock Model-based reinforcement learning with value-targeted regression.
\newblock \textit{arXiv preprint arXiv:2006.01107}.

\bibitem[{Azar et~al.(2012{\natexlab{a}})Azar, G{\'o}mez and
  Kappen}]{azar2012dynamic}
\text{Azar, M.~G.}, \text{G{\'o}mez, V.} and \text{Kappen, H.~J.}
  (2012{\natexlab{a}}).
\newblock Dynamic policy programming.
\newblock \textit{Journal of Machine Learning Research}, \textbf{13}
  3207--3245.

\bibitem[{Azar et~al.(2011)Azar, Munos, Ghavamzadaeh and
  Kappen}]{azar2011speedy}
\text{Azar, M.~G.}, \text{Munos, R.}, \text{Ghavamzadaeh, M.} and \text{Kappen,
  H.~J.} (2011).
\newblock Speedy {Q}-learning.
\newblock In \textit{Advances in Neural Information Processing Systems}.

\bibitem[{Azar et~al.(2012{\natexlab{b}})Azar, Munos and
  Kappen}]{azar2012sample}
\text{Azar, M.~G.}, \text{Munos, R.} and \text{Kappen, B.}
  (2012{\natexlab{b}}).
\newblock On the sample complexity of reinforcement learning with a generative
  model.
\newblock \textit{arXiv preprint arXiv:1206.6461}.

\bibitem[{Azar et~al.(2017)Azar, Osband and Munos}]{azar2017minimax}
\text{Azar, M.~G.}, \text{Osband, I.} and \text{Munos, R.} (2017).
\newblock Minimax regret bounds for reinforcement learning.
\newblock In \textit{International Conference on Machine Learning}.

\bibitem[{Baxter and Bartlett(2000)}]{baxter2000direct}
\text{Baxter, J.} and \text{Bartlett, P.~L.} (2000).
\newblock Direct gradient-based reinforcement learning.
\newblock In \textit{International Symposium on Circuits and Systems}.

\bibitem[{Bhandari and Russo(2019)}]{bhandari2019global}
\text{Bhandari, J.} and \text{Russo, D.} (2019).
\newblock Global optimality guarantees for policy gradient methods.
\newblock \textit{arXiv preprint arXiv:1906.01786}.

\bibitem[{Boyan(2002)}]{boyan2002technical}
\text{Boyan, J.~A.} (2002).
\newblock Least-squares temporal difference learning.
\newblock \textit{Machine Learning}, \textbf{49} 233--246.

\bibitem[{Bradtke and Barto(1996)}]{bradtke1996linear}
\text{Bradtke, S.~J.} and \text{Barto, A.~G.} (1996).
\newblock Linear least-squares algorithms for temporal difference learning.
\newblock \textit{Machine Learning}, \textbf{22} 33--57.

\bibitem[{Bubeck and Cesa-Bianchi(2012)}]{bubeck2012regret}
\text{Bubeck, S.} and \text{Cesa-Bianchi, N.} (2012).
\newblock Regret analysis of stochastic and nonstochastic multi-armed bandit
  problems.
\newblock \textit{Foundations and Trends{\textregistered} in Machine Learning},
  \textbf{5} 1--122.

\bibitem[{Cesa-Bianchi and Lugosi(2006)}]{cesa2006prediction}
\text{Cesa-Bianchi, N.} and \text{Lugosi, G.} (2006).
\newblock \textit{Prediction, Learning, and Games}.
\newblock Cambridge.

\bibitem[{Chen and Jiang(2019)}]{chen2019information}
\text{Chen, J.} and \text{Jiang, N.} (2019).
\newblock Information-theoretic considerations in batch reinforcement learning.
\newblock \textit{arXiv preprint arXiv:1905.00360}.

\bibitem[{Chu et~al.(2011)Chu, Li, Reyzin and Schapire}]{chu2011contextual}
\text{Chu, W.}, \text{Li, L.}, \text{Reyzin, L.} and \text{Schapire, R.}
  (2011).
\newblock Contextual bandits with linear payoff functions.
\newblock In \textit{International Conference on Artificial Intelligence and
  Statistics}.

\bibitem[{Dani et~al.(2008)Dani, Hayes and Kakade}]{dani2008stochastic}
\text{Dani, V.}, \text{Hayes, T.~P.} and \text{Kakade, S.~M.} (2008).
\newblock Stochastic linear optimization under bandit feedback.
\newblock \textit{Conference on Learning Theory}.

\bibitem[{Dann et~al.(2017)Dann, Lattimore and Brunskill}]{dann2017unifying}
\text{Dann, C.}, \text{Lattimore, T.} and \text{Brunskill, E.} (2017).
\newblock Unifying {PAC} and regret: Uniform {PAC} bounds for episodic
  reinforcement learning.
\newblock In \textit{Advances in Neural Information Processing Systems}.

\bibitem[{Dong et~al.(2019)Dong, Peng, Wang and Zhou}]{dong2019sqrt}
\text{Dong, K.}, \text{Peng, J.}, \text{Wang, Y.} and \text{Zhou, Y.} (2019).
\newblock $\sqrt{n}$-regret for learning in {M}arkov decision processes with
  function approximation and low {B}ellman rank.
\newblock \textit{arXiv preprint arXiv:1909.02506}.

\bibitem[{Du et~al.(2019{\natexlab{a}})Du, Kakade, Wang and Yang}]{du2019good}
\text{Du, S.~S.}, \text{Kakade, S.~M.}, \text{Wang, R.} and \text{Yang, L.}
  (2019{\natexlab{a}}).
\newblock Is a good representation sufficient for sample efficient
  reinforcement learning?
\newblock \textit{arXiv preprint arXiv:1910.03016}.

\bibitem[{Du et~al.(2019{\natexlab{b}})Du, Luo, Wang and
  Zhang}]{du2019provably}
\text{Du, S.~S.}, \text{Luo, Y.}, \text{Wang, R.} and \text{Zhang, H.}
  (2019{\natexlab{b}}).
\newblock Provably efficient {Q}-learning with function approximation via
  distribution shift error checking oracle.
\newblock \textit{arXiv preprint arXiv:1906.06321}.

\bibitem[{Duan et~al.(2016)Duan, Chen, Houthooft, Schulman and
  Abbeel}]{duan2016benchmarking}
\text{Duan, Y.}, \text{Chen, X.}, \text{Houthooft, R.}, \text{Schulman, J.} and
  \text{Abbeel, P.} (2016).
\newblock Benchmarking deep reinforcement learning for continuous control.
\newblock In \textit{International Conference on Machine Learning}.

\bibitem[{Even-Dar et~al.(2009)Even-Dar, Kakade and Mansour}]{even2009online}
\text{Even-Dar, E.}, \text{Kakade, S.~M.} and \text{Mansour, Y.} (2009).
\newblock Online {M}arkov decision processes.
\newblock \textit{Mathematics of Operations Research}, \textbf{34} 726--736.

\bibitem[{Farahmand et~al.(2010)Farahmand, Szepesv{\'a}ri and
  Munos}]{farahmand2010error}
\text{Farahmand, A.-m.}, \text{Szepesv{\'a}ri, C.} and \text{Munos, R.} (2010).
\newblock Error propagation for approximate policy and value iteration.
\newblock In \textit{Advances in Neural Information Processing Systems}.

\bibitem[{Fazel et~al.(2018)Fazel, Ge, Kakade and Mesbahi}]{fazel2018global}
\text{Fazel, M.}, \text{Ge, R.}, \text{Kakade, S.~M.} and \text{Mesbahi, M.}
  (2018).
\newblock Global convergence of policy gradient methods for the linear
  quadratic regulator.
\newblock \textit{arXiv preprint arXiv:1801.05039}.

\bibitem[{Geist et~al.(2019)Geist, Scherrer and Pietquin}]{geist2019theory}
\text{Geist, M.}, \text{Scherrer, B.} and \text{Pietquin, O.} (2019).
\newblock A theory of regularized {M}arkov decision processes.
\newblock \textit{arXiv preprint arXiv:1901.11275}.

\bibitem[{Jaksch et~al.(2010)Jaksch, Ortner and Auer}]{jaksch2010near}
\text{Jaksch, T.}, \text{Ortner, R.} and \text{Auer, P.} (2010).
\newblock Near-optimal regret bounds for reinforcement learning.
\newblock \textit{Journal of Machine Learning Research}, \textbf{11}
  1563--1600.

\bibitem[{Jiang et~al.(2017)Jiang, Krishnamurthy, Agarwal, Langford and
  Schapire}]{jiang2017contextual}
\text{Jiang, N.}, \text{Krishnamurthy, A.}, \text{Agarwal, A.}, \text{Langford,
  J.} and \text{Schapire, R.~E.} (2017).
\newblock Contextual decision processes with low {B}ellman rank are
  {PAC}-learnable.
\newblock In \textit{International Conference on Machine Learning}.

\bibitem[{Jin et~al.(2018)Jin, Allen-Zhu, Bubeck and Jordan}]{jin2018q}
\text{Jin, C.}, \text{Allen-Zhu, Z.}, \text{Bubeck, S.} and \text{Jordan,
  M.~I.} (2018).
\newblock Is {Q}-learning provably efficient?
\newblock In \textit{Advances in Neural Information Processing Systems}.

\bibitem[{Jin et~al.(2019)Jin, Yang, Wang and Jordan}]{jin2019provably}
\text{Jin, C.}, \text{Yang, Z.}, \text{Wang, Z.} and \text{Jordan, M.~I.}
  (2019).
\newblock Provably efficient reinforcement learning with linear function
  approximation.
\newblock \textit{arXiv preprint arXiv:1907.05388}.

\bibitem[{Kakade(2002)}]{kakade2002natural}
\text{Kakade, S.~M.} (2002).
\newblock A natural policy gradient.
\newblock In \textit{Advances in Neural Information Processing Systems}.

\bibitem[{Kakade(2003)}]{kakade2003sample}
\text{Kakade, S.~M.} (2003).
\newblock \textit{On the Sample Complexity of Reinforcement Learning}.
\newblock Ph.D. thesis, University of London.

\bibitem[{Koenig and Simmons(1993)}]{koenig1993complexity}
\text{Koenig, S.} and \text{Simmons, R.~G.} (1993).
\newblock Complexity analysis of real-time reinforcement learning.
\newblock In \textit{Association for the Advancement of Artificial
  Intelligence}.

\bibitem[{Konda and Tsitsiklis(2000)}]{konda2000actor}
\text{Konda, V.~R.} and \text{Tsitsiklis, J.~N.} (2000).
\newblock Actor-critic algorithms.
\newblock In \textit{Advances in Neural Information Processing Systems}.

\bibitem[{Lattimore and Szepesvari(2019)}]{lattimore2019learning}
\text{Lattimore, T.} and \text{Szepesvari, C.} (2019).
\newblock Learning with good feature representations in bandits and in {RL}
  with a generative model.
\newblock \textit{arXiv preprint arXiv:1911.07676}.

\bibitem[{Leffler et~al.(2007)Leffler, Littman and
  Edmunds}]{leffler2007efficient}
\text{Leffler, B.~R.}, \text{Littman, M.~L.} and \text{Edmunds, T.} (2007).
\newblock Efficient reinforcement learning with relocatable action models.
\newblock In \textit{Association for the Advancement of Artificial
  Intelligence}.

\bibitem[{Liu et~al.(2019)Liu, Cai, Yang and Wang}]{liu2019neural}
\text{Liu, B.}, \text{Cai, Q.}, \text{Yang, Z.} and \text{Wang, Z.} (2019).
\newblock Neural proximal/trust region policy optimization attains globally
  optimal policy.
\newblock \textit{arXiv preprint arXiv:1906.10306}.

\bibitem[{Mania et~al.(2018)Mania, Guy and Recht}]{mania2018simple}
\text{Mania, H.}, \text{Guy, A.} and \text{Recht, B.} (2018).
\newblock Simple random search provides a competitive approach to reinforcement
  learning.
\newblock \textit{arXiv preprint arXiv:1803.07055}.

\bibitem[{Munos and Szepesv{\'a}ri(2008)}]{munos2008finite}
\text{Munos, R.} and \text{Szepesv{\'a}ri, C.} (2008).
\newblock Finite-time bounds for fitted value iteration.
\newblock \textit{Journal of Machine Learning Research}, \textbf{9} 815--857.

\bibitem[{Nemirovsky and Yudin(1983)}]{nemirovsky1983problem}
\text{Nemirovsky, A.~S.} and \text{Yudin, D.~B.} (1983).
\newblock \textit{Problem Complexity and Method Efficiency in Optimization.}
\newblock Wiley.

\bibitem[{Neu et~al.(2010{\natexlab{a}})Neu, Antos, Gy{\"o}rgy and
  Szepesv{\'a}ri}]{neu2010bandit}
\text{Neu, G.}, \text{Antos, A.}, \text{Gy{\"o}rgy, A.} and
  \text{Szepesv{\'a}ri, C.} (2010{\natexlab{a}}).
\newblock Online {M}arkov decision processes under bandit feedback.
\newblock In \textit{Advances in Neural Information Processing Systems}.

\bibitem[{Neu et~al.(2010{\natexlab{b}})Neu, Gy{\"o}rgy and
  Szepesv{\'a}ri}]{neu2010online}
\text{Neu, G.}, \text{Gy{\"o}rgy, A.} and \text{Szepesv{\'a}ri, C.}
  (2010{\natexlab{b}}).
\newblock The online loop-free stochastic shortest-path problem.
\newblock In \textit{Conference on Learning Theory}.

\bibitem[{Neu et~al.(2012)Neu, Gy{\"o}rgy and
  Szepesv{\'a}ri}]{neu2012adversarial}
\text{Neu, G.}, \text{Gy{\"o}rgy, A.} and \text{Szepesv{\'a}ri, C.} (2012).
\newblock The adversarial stochastic shortest path problem with unknown
  transition probabilities.
\newblock In \textit{International Conference on Artificial Intelligence and
  Statistics}.

\bibitem[{Neu et~al.(2017)Neu, Jonsson and G{\'o}mez}]{neu2017unified}
\text{Neu, G.}, \text{Jonsson, A.} and \text{G{\'o}mez, V.} (2017).
\newblock A unified view of entropy-regularized {M}arkov decision processes.
\newblock \textit{arXiv preprint arXiv:1705.07798}.

\bibitem[{OpenAI(2019)}]{openai2019dota}
\text{OpenAI} (2019).
\newblock Open{AI} {F}ive.
\newblock \url{https://openai.com/five/}.

\bibitem[{Osband and Van~Roy(2016)}]{osband2016lower}
\text{Osband, I.} and \text{Van~Roy, B.} (2016).
\newblock On lower bounds for regret in reinforcement learning.
\newblock \textit{arXiv preprint arXiv:1608.02732}.

\bibitem[{Osband et~al.(2014)Osband, Van~Roy and
  Wen}]{osband2014generalization}
\text{Osband, I.}, \text{Van~Roy, B.} and \text{Wen, Z.} (2014).
\newblock Generalization and exploration via randomized value functions.
\newblock \textit{arXiv preprint arXiv:1402.0635}.

\bibitem[{Rosenberg and Mansour(2019{\natexlab{a}})}]{rosenberg2019online}
\text{Rosenberg, A.} and \text{Mansour, Y.} (2019{\natexlab{a}}).
\newblock Online convex optimization in adversarial {M}arkov decision
  processes.
\newblock \textit{arXiv preprint arXiv:1905.07773}.

\bibitem[{Rosenberg and Mansour(2019{\natexlab{b}})}]{rosenberg2019bandit}
\text{Rosenberg, A.} and \text{Mansour, Y.} (2019{\natexlab{b}}).
\newblock Online stochastic shortest path with bandit feedback and unknown
  transition function.
\newblock In \textit{Advances in Neural Information Processing Systems}.

\bibitem[{Rusmevichientong and Tsitsiklis(2010)}]{rusmevichientong2010linearly}
\text{Rusmevichientong, P.} and \text{Tsitsiklis, J.~N.} (2010).
\newblock Linearly parameterized bandits.
\newblock \textit{Mathematics of Operations Research}, \textbf{35} 395--411.

\bibitem[{Schulman et~al.(2015)Schulman, Levine, Abbeel, Jordan and
  Moritz}]{schulman2015trust}
\text{Schulman, J.}, \text{Levine, S.}, \text{Abbeel, P.}, \text{Jordan, M.}
  and \text{Moritz, P.} (2015).
\newblock Trust region policy optimization.
\newblock In \textit{International Conference on Machine Learning}.

\bibitem[{Schulman et~al.(2017)Schulman, Wolski, Dhariwal, Radford and
  Klimov}]{schulman2017proximal}
\text{Schulman, J.}, \text{Wolski, F.}, \text{Dhariwal, P.}, \text{Radford, A.}
  and \text{Klimov, O.} (2017).
\newblock Proximal policy optimization algorithms.
\newblock \textit{arXiv preprint arXiv:1707.06347}.

\bibitem[{Sidford et~al.(2018{\natexlab{a}})Sidford, Wang, Wu, Yang and
  Ye}]{sidford2018near}
\text{Sidford, A.}, \text{Wang, M.}, \text{Wu, X.}, \text{Yang, L.} and
  \text{Ye, Y.} (2018{\natexlab{a}}).
\newblock Near-optimal time and sample complexities for solving {M}arkov
  decision processes with a generative model.
\newblock In \textit{Advances in Neural Information Processing Systems}.

\bibitem[{Sidford et~al.(2018{\natexlab{b}})Sidford, Wang, Wu and
  Ye}]{sidford2018variance}
\text{Sidford, A.}, \text{Wang, M.}, \text{Wu, X.} and \text{Ye, Y.}
  (2018{\natexlab{b}}).
\newblock Variance reduced value iteration and faster algorithms for solving
  {M}arkov decision processes.
\newblock In \textit{Symposium on Discrete Algorithms}.

\bibitem[{Silver et~al.(2016)Silver, Huang, Maddison, Guez, Sifre, Van
  Den~Driessche, Schrittwieser, Antonoglou, Panneershelvam, Lanctot
  et~al.}]{silver2016mastering}
\text{Silver, D.}, \text{Huang, A.}, \text{Maddison, C.~J.}, \text{Guez, A.},
  \text{Sifre, L.}, \text{Van Den~Driessche, G.}, \text{Schrittwieser, J.},
  \text{Antonoglou, I.}, \text{Panneershelvam, V.}, \text{Lanctot, M.}
  \text{et~al.} (2016).
\newblock Mastering the game of {G}o with deep neural networks and tree search.
\newblock \textit{Nature}, \textbf{529} 484.

\bibitem[{Silver et~al.(2017)Silver, Schrittwieser, Simonyan, Antonoglou,
  Huang, Guez, Hubert, Baker, Lai, Bolton et~al.}]{silver2017mastering}
\text{Silver, D.}, \text{Schrittwieser, J.}, \text{Simonyan, K.},
  \text{Antonoglou, I.}, \text{Huang, A.}, \text{Guez, A.}, \text{Hubert, T.},
  \text{Baker, L.}, \text{Lai, M.}, \text{Bolton, A.} \text{et~al.} (2017).
\newblock Mastering the game of {G}o without human knowledge.
\newblock \textit{Nature}, \textbf{550} 354.

\bibitem[{Strehl et~al.(2006)Strehl, Li, Wiewiora, Langford and
  Littman}]{strehl2006pac}
\text{Strehl, A.~L.}, \text{Li, L.}, \text{Wiewiora, E.}, \text{Langford, J.}
  and \text{Littman, M.~L.} (2006).
\newblock {PAC} model-free reinforcement learning.
\newblock In \textit{International Conference on Machine Learning}.

\bibitem[{Sutton and Barto(2018)}]{sutton2018reinforcement}
\text{Sutton, R.~S.} and \text{Barto, A.~G.} (2018).
\newblock \textit{Reinforcement Learning: An Introduction}.
\newblock MIT.

\bibitem[{Sutton et~al.(2000)Sutton, McAllester, Singh and
  Mansour}]{sutton2000policy}
\text{Sutton, R.~S.}, \text{McAllester, D.~A.}, \text{Singh, S.~P.} and
  \text{Mansour, Y.} (2000).
\newblock Policy gradient methods for reinforcement learning with function
  approximation.
\newblock In \textit{Advances in Neural Information Processing Systems}.

\bibitem[{Tosatto et~al.(2017)Tosatto, Pirotta, D'Eramo and
  Restelli}]{tosatto2017boosted}
\text{Tosatto, S.}, \text{Pirotta, M.}, \text{D'Eramo, C.} and \text{Restelli,
  M.} (2017).
\newblock Boosted fitted {Q}-iteration.
\newblock In \textit{International Conference on Machine Learning}.

\bibitem[{Van~Roy and Dong(2019)}]{van2019comments}
\text{Van~Roy, B.} and \text{Dong, S.} (2019).
\newblock Comments on the {D}u-{K}akade-{W}ang-{Y}ang lower bounds.
\newblock \textit{arXiv preprint arXiv:1911.07910}.

\bibitem[{Wainwright(2019)}]{wainwright2019variance}
\text{Wainwright, M.~J.} (2019).
\newblock Variance-reduced {Q}-learning is minimax optimal.
\newblock \textit{arXiv preprint arXiv:1906.04697}.

\bibitem[{Wang et~al.(2019)Wang, Cai, Yang and Wang}]{wang2019neural}
\text{Wang, L.}, \text{Cai, Q.}, \text{Yang, Z.} and \text{Wang, Z.} (2019).
\newblock Neural policy gradient methods: Global optimality and rates of
  convergence.
\newblock \textit{arXiv preprint arXiv:1909.01150}.

\bibitem[{Wang et~al.(2018)Wang, Li and He}]{wang2018deep}
\text{Wang, W.~Y.}, \text{Li, J.} and \text{He, X.} (2018).
\newblock Deep reinforcement learning for {NLP}.
\newblock In \textit{Association for Computational Linguistics}.

\bibitem[{Wen and Van~Roy(2017)}]{wen2017efficient}
\text{Wen, Z.} and \text{Van~Roy, B.} (2017).
\newblock Efficient reinforcement learning in deterministic systems with value
  function generalization.
\newblock \textit{Mathematics of Operations Research}, \textbf{42} 762--782.

\bibitem[{Williams(1992)}]{williams1992simple}
\text{Williams, R.~J.} (1992).
\newblock Simple statistical gradient-following algorithms for connectionist
  reinforcement learning.
\newblock \textit{Machine Learning}, \textbf{8} 229--256.

\bibitem[{Xiao(2010)}]{xiao2010dual}
\text{Xiao, L.} (2010).
\newblock Dual averaging methods for regularized stochastic learning and online
  optimization.
\newblock \textit{Journal of Machine Learning Research}, \textbf{11}
  2543--2596.

\bibitem[{Yang and Wang(2019{\natexlab{a}})}]{yang2019reinforcement}
\text{Yang, L.} and \text{Wang, M.} (2019{\natexlab{a}}).
\newblock Reinforcement leaning in feature space: Matrix bandit, kernels, and
  regret bound.
\newblock \textit{arXiv preprint arXiv:1905.10389}.

\bibitem[{Yang and Wang(2019{\natexlab{b}})}]{yang2019sample}
\text{Yang, L.} and \text{Wang, M.} (2019{\natexlab{b}}).
\newblock Sample-optimal parametric {Q}-learning using linearly additive
  features.
\newblock In \textit{International Conference on Machine Learning}.

\bibitem[{Yang et~al.(2019{\natexlab{a}})Yang, Chen, Hong and
  Wang}]{yang2019global}
\text{Yang, Z.}, \text{Chen, Y.}, \text{Hong, M.} and \text{Wang, Z.}
  (2019{\natexlab{a}}).
\newblock On the global convergence of actor-critic: A case for linear
  quadratic regulator with ergodic cost.
\newblock \textit{arXiv preprint arXiv:1907.06246}.

\bibitem[{Yang et~al.(2019{\natexlab{b}})Yang, Xie and
  Wang}]{yang2019theoretical}
\text{Yang, Z.}, \text{Xie, Y.} and \text{Wang, Z.} (2019{\natexlab{b}}).
\newblock A theoretical analysis of deep {Q}-learning.
\newblock \textit{arXiv preprint arXiv:1901.00137}.

\bibitem[{Yu et~al.(2009)Yu, Mannor and Shimkin}]{yu2009markov}
\text{Yu, J.~Y.}, \text{Mannor, S.} and \text{Shimkin, N.} (2009).
\newblock {M}arkov decision processes with arbitrary reward processes.
\newblock \textit{Mathematics of Operations Research}, \textbf{34} 737--757.

\bibitem[{Zhou et~al.(2020)Zhou, He and Gu}]{zhou2020provably}
\text{Zhou, D.}, \text{He, J.} and \text{Gu, Q.} (2020).
\newblock Provably efficient reinforcement learning for discounted mdps with
  feature mapping.
\newblock \textit{arXiv preprint arXiv:2006.13165}.

\bibitem[{Zimin and Neu(2013)}]{zimin2013online}
\text{Zimin, A.} and \text{Neu, G.} (2013).
\newblock Online learning in episodic {M}arkovian decision processes by
  relative entropy policy search.
\newblock In \textit{Advances in Neural Information Processing Systems}.

\end{thebibliography}
